%% file: main.tex
\documentclass{article}
\usepackage{graphicx}
\usepackage{authblk}

\usepackage{algorithm}
\usepackage{basicreq}
\usepackage{math_notations}
\usepackage{times}
\usepackage{xcolor}
\usepackage[round]{natbib}

\title{Towards Reliable, Uncertainty-Aware Alignment}
\author[1]{Debangshu Banerjee}
\author[2]{Kintan Saha}
\author[3]{Aditya Gopalan}
\affil[1,3]{ Department of Electrical and Communication Engineering, Indian Institute of Science, India} 
\affil[2]{Undergraduate Programme, Indian Institute of Science, India}

\begin{document}

\maketitle

\begin{abstract}
Alignment of large language models (LLMs) typically involves training a reward model on preference data, followed by policy optimization with respect to the reward model. However, optimizing policies with respect to a {\em single} reward model estimate can render it vulnerable to inaccuracies in the reward model. We empirically study the variability of reward model training on open-source benchmarks. We observe that independently trained reward models on the same preference dataset can exhibit substantial disagreement, highlighting the instability of current alignment strategies. Employing a theoretical model, we demonstrate that variability in reward model estimation can cause overfitting, leading to the risk of performance degradation. To mitigate this risk, we propose a variance-aware policy optimization framework for preference-based alignment. The key ingredient of the framework is a new policy regularizer that incorporates reward model variance estimates. We show that variance-aware policy optimization provably reduces the risk of outputting a worse policy than the default. Experiments across diverse LLM and reward model configurations confirm that our approach yields more stable and robust alignment than the standard (variance-unaware) pipeline.
\end{abstract}

%%%%%%%%%%%%%%%%%%%%%%%%%%%%%%%%%%%%%%%%%%%%%%%%%
\input{Neurips/Introduction}
\input{Neurips/Theory}
\input{Neurips/Experiments}

\newpage
\bibliographystyle{plainnat}
\bibliography{References}

\newpage
\appendix
\input{Neurips/Proofs}

\input{Neurips/Reward_Model}

\input{Neurips/Reward_Prompting}

\input{Neurips/PPO}

\input{Neurips/PPO_Hardware}

\input{Neurips/Limitations}
%%%%%%%%%%%

% %%%%%%%%%%%%%%%%%%%%%%%%%%%%%%%%%%%%%%%%%%%%%%%%%%%%%%%%%%%%

% \bibliographystyle{plainnat}
% \bibliography{References}

% %%%%%%%%%%%%%%%%%%%%%%%%%%%%%%%%%%%%%%%%%%%%%%%%%%%%%%%%%%%%

% \input{ARXIV/Appendix}

\end{document}

%% file: Neurips/Introduction.tex
\section{Introduction}

\begin{figure}[!htbp]
    \centering
    \begin{minipage}{0.45\textwidth} % Adjust width as needed
        \centering
        \begin{tikzpicture}[
            scale = 0.70, % Adjust scale to make it smaller
            transform shape,
            node distance=0.5cm,
            every node/.style={rectangle, draw, align=center, text width=7cm},
            prompt/.style={fill=blue!20},
            response/.style={fill=green!20}
            ]

            % Define the prompt node
            \node[prompt] (prompt) {%
            \textbf{Prompt:}\\[5pt]
            \emph{Who created the Superman cartoon character?}
            };

            % Define the response node below the prompt
            \node[response, below=of prompt] (response) {%
            \textbf{Response:}\\[0.001pt]
            \emph{Superman, the iconic comic book superhero, was created by writer Jerry Siegel and artist Joe Shuster. Superman first appeared in Action Comics \#1, which was published by Detective Comics, Inc. (later DC Comics) in June 1938.}
            };
        \end{tikzpicture}
    \end{minipage}%
    \hspace{0.04\textwidth} % Small space between the figures
    \begin{minipage}{0.45\linewidth}
        \centering
        \includegraphics[width=0.8\linewidth]{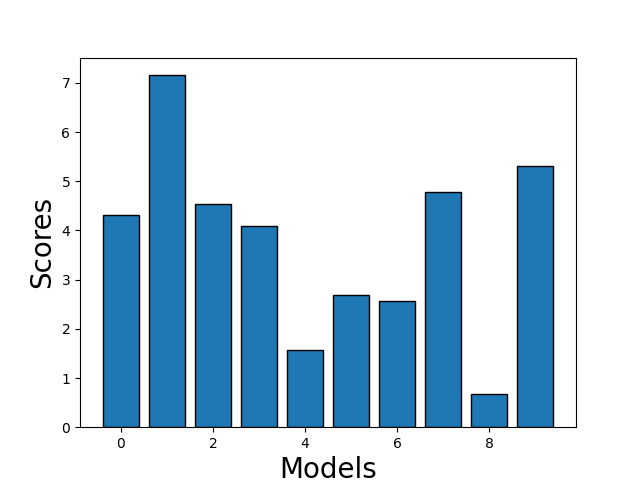} % Adjust width for size
    \end{minipage}
    \caption{\footnotesize{Reward scores assigned by $10$ reward models on the same prompt-response pair. The reward models are identical in that they are trained independently on the same dataset, with the same hyperparameters and number of epochs. Despite this, we see a wide variation in the score assigned by each model. }}
    \label{figure:evidence}
\end{figure}
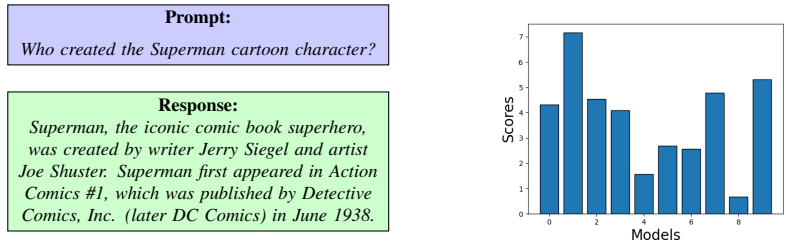

Reinforcement Learning with Human Feedback (RLHF)~\citep{christiano2017deep, ziegler2019fine} is a fine-tuning method for aligning large language models (LLMs) with human preferences. It is a key process in the post-training of systems such as \textsc{ChatGPT}~\citep{openai2023gpt}, \textsc{Claude}~\citep{claude_3}, \textsc{Gemini}~\citep{team2023gemini}, and \textsc{LLaMA-3}~\citep{llama_cite}. RLHF aims to achieve model alignment by guiding LLMs toward human-defined ethical, safety, and utility objectives. In the standard RLHF framework~\citep{ouyang2022training, bai2022constitutional, touvron2023llama}, a reward model predicts human preferences, serving as a proxy for human judgment and providing the reward signal for downstream policy optimization. 

\paragraph{Challenges of Reward Model Reliability}
Reward models in RLHF can be inherently noisy, posing challenges for robust alignment. As our experiments show, ten independently-trained, identically configured reward models assign varying scores to the same prompt-response pair (Figure~\ref{figure:evidence}). This variability stems from three key factors. First, reward models are trained on far less data, typically hundreds of thousands of preference pairs, than LLMs, limiting generalization and introducing statistical noise. Second, stochastic optimization techniques, such as mini-batch gradient descent with few training epochs~\citep{stiennon2020learning, meta2024introducing}, further increase randomness across trained instances. Lastly, in the case where absolute reward scores, instead of relative preferences, are directly elicited from a population (of humans or AI agents) for downstream usage in reward model training or policy optimization, score variability naturally arises due to diversity of opinions and biases. These factors imply that reward models are not reliable oracles, and overfitting to their noisy outputs may lead to alignment with spurious artifacts rather than true human intent.

\paragraph{Contributions} 
This work makes three primary contributions. First, we empirically demonstrate inconsistencies across identically configured reward models on open-source datasets, revealing inherent uncertainty in reward modeling (Appendix~\ref{sec:reward_model}). Second, we introduce a mathematical model of uncertain rewards and propose a conservative, variance-aware policy optimization framework that provably reduces risk under uncertainty (Section~\ref{sec:theory}). Finally, we present theoretical guarantees and numerical evidence showing that our approach mitigates policy overfitting and stabilizes reward outcomes (Section~\ref{sec:experiment}).

\paragraph{Example Illustration} To illustrate the importance of accounting for uncertainty in reward models, consider a simple bandit problem with 3 arms, for which any %
% Aligning a language model can be viewed through the lens of contextual bandits, where 
policy (`LLM') assigns probabilities to different arms ('responses') to maximize the expected return (although an LLM represents a contextual bandit policy, we consider a simplified example without contexts to make our case). In this example, the true rewards $r_1^*, r_2^*, r_3^*$ (green circles in Figure~\ref{fig:example}) satisfy $r_1^* < r_2^* < r_3^*$. However, the estimated rewards $\hat{R}_1, \hat{R}_2, \hat{R}_3$ (blue circles) suggest the opposite ranking, with Arm 1 appearing most favorable due to a spurious high estimate. If a policy is optimized purely based on these estimates, it assigns the highest probability to Arm 1, leading to a reduced true return. Crucially, the uncertainty intervals (red brackets) reveal that the reward estimate for Arm 1 is highly uncertain, while Arms 2 and 3 offer more stable, albeit lower, estimates. A more discerning, variance-aware policy improvement strategy ought to recognize this and assign lower weight to the unreliable high estimate of Arm 1, instead favoring the more dependable estimates of Arms 2 and 3. To quantify this effect, Figure~\ref{fig:example-soln} plots the expected return $\sum_{i=1}^3 \pi_i r_i^*$ over $1000$ samples $\hat{R}_i\sim \cN(r^*_i, \sigma_i^2)$, for $i =1,2,3$. Policies learned without variance information (cyan distribution) exhibit a $39\%$ chance of underperforming the initial uniform policy, whereas this probability drops to $0.15\%$ when variance is incorporated (red distribution). The base policy is a uniform distribution over the three arms and is marked as a purple vertical line. These results underscore the value of modeling uncertainty in policy optimization, especially in noisy alignment settings.

\begin{figure}[!htbp]
\begin{minipage}{0.45\linewidth}
\begin{tikzpicture}[scale=0.75, every node/.style={scale=0.75}]
\def\armspacing{2.5}

% Axes
\draw[->, thick] (-1, 0) -- (2*\armspacing+1, 0) node[right]{};
\draw[->, thick] (-1, 0) -- (-1, 4.5) node[above] {Reward};

% Y-axis ticks and labels
\foreach \y in {0,1,...,4} {
    \draw (-1.01, \y) -- (-0.9, \y);
    \node[left] at (-1, \y) {\small{\y}};
}

% Arm 1
\draw[black] (0*\armspacing, 0) -- (0*\armspacing, 4);
\draw[red] (-0.2+\armspacing*0, 0.5) -- (0.2+\armspacing*0, 0.5);
\draw[red] (-0.2+\armspacing*0, 3.5) -- (0.2+\armspacing*0, 3.5);
\draw[red] (-0.2+\armspacing*0, 0.5) -- (-0.2+\armspacing*0, 0.7);
\draw[red] (0.2+\armspacing*0, 0.5) -- (0.2+\armspacing*0, 0.7);
\draw[red] (-0.2+\armspacing*0, 3.5) -- (-0.2+\armspacing*0, 3.3);
\draw[red] (0.2+\armspacing*0, 3.5) -- (0.2+\armspacing*0, 3.3);
\fill[blue] (0*\armspacing, 3.2) circle (0.07);
\node[blue, right] at (0.2+\armspacing*0, 3.2) {\small{$\hat{R}_1$}};
\fill[green!50!blue] (0*\armspacing, 1.0) circle (0.07);
\node[green!50!blue, left] at (-0.2+\armspacing*0, 1.0) {$r^*_1$};
\node[black] at (0*\armspacing, -0.5) {\footnotesize{Response/Arm 1}};

% Arm 2
\draw[black] (1*\armspacing, 0) -- (1*\armspacing, 4);
\draw[red] (-0.2+\armspacing*1, 1.6) -- (0.2+\armspacing*1, 1.6);
\draw[red] (-0.2+\armspacing*1, 2.4) -- (0.2+\armspacing*1, 2.4);
\draw[red] (-0.2+\armspacing*1, 1.6) -- (-0.2+\armspacing*1, 1.7);
\draw[red] (0.2+\armspacing*1, 1.6) -- (0.2+\armspacing*1, 1.7);
\draw[red] (-0.2+\armspacing*1, 2.4) -- (-0.2+\armspacing*1, 2.3);
\draw[red] (0.2+\armspacing*1, 2.4) -- (0.2+\armspacing*1, 2.3);
\fill[blue] (1*\armspacing, 2.2) circle (0.07);
\node[blue, right] at (0.2+\armspacing*1, 2.2) {\small{$\hat{R}_2$}};
\fill[green!50!blue] (1*\armspacing, 1.8) circle (0.07);
\node[green!50!blue, left] at (-0.2+\armspacing*1, 1.8) {$r^*_2$};
\node[black] at (1*\armspacing, -0.5) {\footnotesize{Response/Arm 2}};

% Arm 3
\draw[black] (2*\armspacing, 0) -- (2*\armspacing, 4);
\draw[red] (-0.2+\armspacing*2, 1.2) -- (0.2+\armspacing*2, 1.2);
\draw[red] (-0.2+\armspacing*2, 1.8) -- (0.2+\armspacing*2, 1.8);
\draw[red] (-0.2+\armspacing*2, 1.2) -- (-0.2+\armspacing*2, 1.25);
\draw[red] (0.2+\armspacing*2, 1.2) -- (0.2+\armspacing*2, 1.25);
\draw[red] (-0.2+\armspacing*2, 1.8) -- (-0.2+\armspacing*2, 1.75);
\draw[red] (0.2+\armspacing*2, 1.8) -- (0.2+\armspacing*2, 1.75);
\fill[green!50!blue] (2*\armspacing, 1.65) circle (0.07);
\node[green!50!blue, left] at (-0.2+\armspacing*2, 1.65) {$r_3^*$};
\fill[blue] (2*\armspacing, 1.35) circle (0.07);
\node[blue, right] at (0.2+\armspacing*2, 1.35) {\small{$\hat{R}_3$}};
\node[black] at (2*\armspacing, -0.5) {\footnotesize{Response/Arm 3}};

\end{tikzpicture}
\caption{\footnotesize A 3-armed bandit problem illustrating true rewards $r_1^*, r_2^*, r_3^*$ (green circles), estimated rewards $\hat{R}_1, \hat{R}_2, \hat{R}_3$ (blue circles), and their uncertainty intervals (red brackets). Although Arm 1 has the lowest true reward, it yields the highest estimate $\hat{R}_1$, while Arms 2 and 3 have more accurate but lower estimates. A naive policy that relies solely on the estimates would overemphasize Arm 1, resulting in a lower expected return. In contrast, a variance-aware strategy that accounts for the higher uncertainty in $\hat{R}_1$ assigns it a lower probability, leading to a more robust and higher average return. 
}
\label{fig:example}
\end{minipage}
\hfill
\begin{minipage}{0.51\linewidth}
    \includegraphics[width=1.08\linewidth]{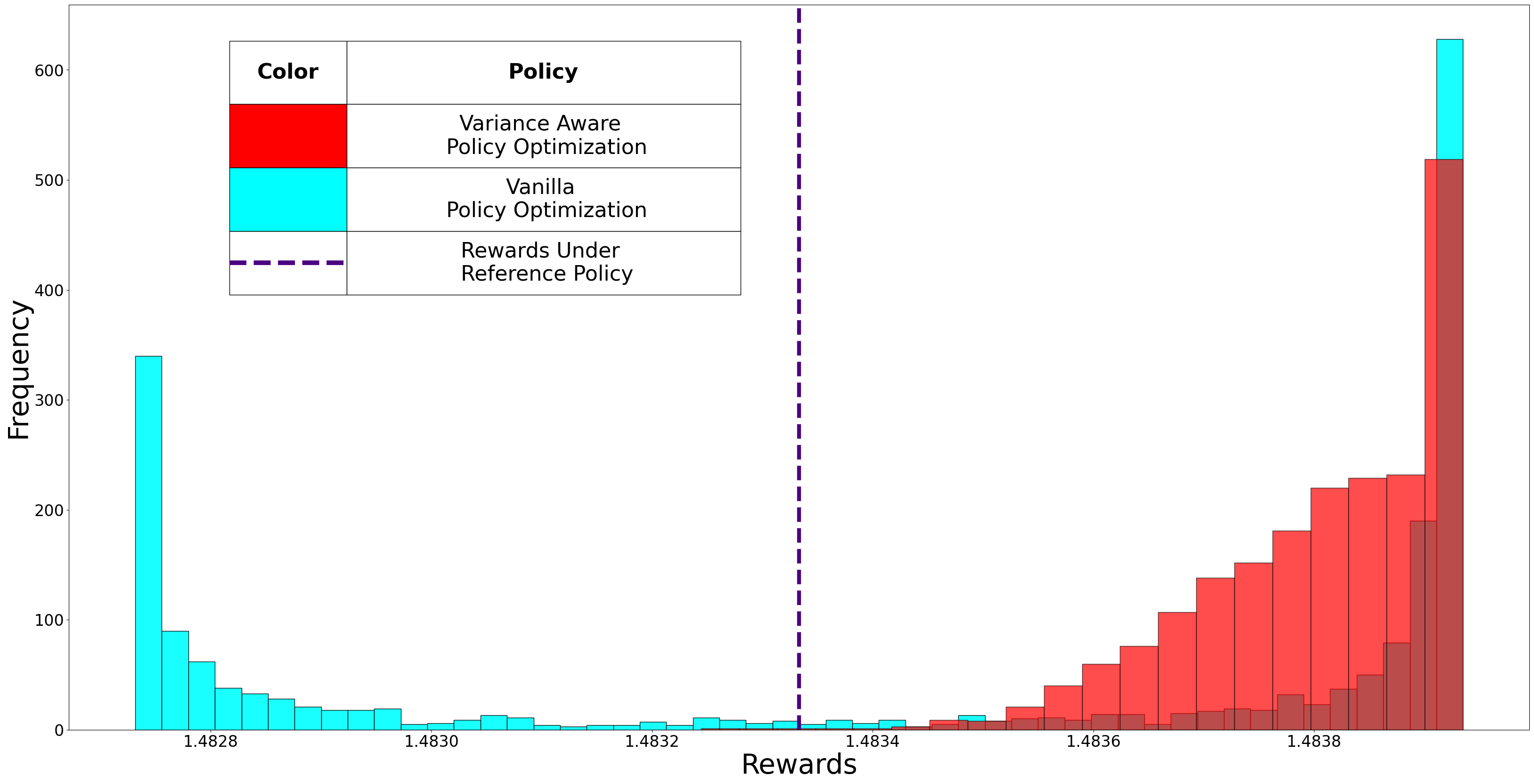}
    \caption{\footnotesize
We plot the expected return $\sum_{i=1}^3 \pi_i r^*_i$ for the 3-armed bandit example shown in Figure \ref{fig:example}, where the probabilities $\{\pi_i\}_{i=1}^3$ are computed based on observed estimates $\{\hat{R}_i\}_{i=1}^3$. We generate 1000 normally distributed samples of $\{\hat{R}_i\}$, and for each sample, derive a corresponding probability vector $\pi$. The cyan-colored distribution represents the returns when $\pi$ is chosen without considering reward uncertainty, while the red distribution accounts for variance information. The initial policy is uniform over the three arms, yielding an average return of $1.4833$. Without using variance information, there is a $39\%$ chance of performing worse than the initial policy, whereas this probability drops to $0.15\%$ when variance-aware probabilities are used. This highlights the reduced risk associated with incorporating uncertainty into policy selection.
}
\label{fig:example-soln}
\end{minipage}
\end{figure}

\paragraph{Related Work.} 
Reward models in RLHF approximate human preferences but are vulnerable to failures like \textit{reward hacking}~\citep{amodei2016concrete} and \textit{overoptimization}~\citep{gao2023scaling}, especially under limited data and stochastic training. To address this, prior work has proposed ensemble methods~\citep{coste2023reward, eisenstein2023helping, rame2024warm, zhang2024improving}, conservative objectives using lower confidence bounds (LCBs)~\citep{liang2022reward, zhai2023uncertainty, xiong2024iterative}, and uncertainty modeling via Bayesian dropout~\citep{gal2016dropout} or reparameterization~\citep{lou2024uncertainty}. Recent approaches bypass separate reward models by prompting LLMs directly~\citep{kwon2023reward, llmasajudge2023, singla-etal-2024-dynamic}, or by eliciting uncertainty estimates from LLMs via confidence scores~\citep{xiong2023can}, rationality judgments~\citep{tanneru2023quantifying}, ensemble variance~\citep{zhang2024bayesian}, or belief updates~\citep{liu2024iterative}. We build on these insights by explicitly incorporating reward uncertainty—via ensembles or prompts—into policy optimization to improve robustness.

%% file: Neurips/Theory.tex
\section{Mathematical Modeling}
\label{sec:theory}
\paragraph{Notations:} 
We assume that prompts are strings $x$ from a prompt set $\mathcal{X}$, and responses are strings $y$ from a response set $\mathcal{Y}$\footnote{Typically, $\mathcal{X}$ and $\mathcal{Y}$ both equal the space of all finite-length token sequences in an LLM context, but they can differ depending on the  application.}. A reward model assigns a scalar value to each prompt–response pair, i.e., $\hat{R} : \mathcal{X} \times \mathcal{Y} \to \mathbb{R}$. We model a large language model (LLM) as a policy $\pi$ that defines a probability distribution over responses given a prompt, i.e., $\pi(\cdot \mid x) \in \cP({\cY})$. We denote by $\rho_{\cX}$ the distribution over prompts. With a slight abuse of notation, we treat the policy $\pi$ as an induced joint distribution over $(x, y)$ pairs, which allows us to write the expected reward compactly as $\hat{R}^\top \pi$. For any vector $d$, we define the $\Sigma$-weighted norm as $\|d\|_{\Sigma} := \sqrt{d^\top \Sigma d}$, where $\Sigma$ is a positive semi-definite matrix.

\begin{restatable}[\textbf{Noisy Reward Model}]{assumption}{noisyreward}
\label{assm:gaussian}
We assume that $\hat{R}$ is a noisy observation of the true reward function $r^*$. For any $(x, y)$, the observed reward $\hat{R}(x, y)$ is modeled as a Gaussian perturbation of $r^*(x, y)$:
\[
\hat{R}(x, y) = r^*(x, y) + \mathcal{N}\big(0, \sigma^2(x, y)\big),
\]
where $\mathcal{N}(0, \sigma^2(x, y))$ is a Gaussian random variable with mean zero and variance $\sigma^2(x, y)$. The values $\hat{R}(x, y)$ are assumed to be independent across different $(x, y)$.

This implies that $\hat{R} \sim \mathcal{N}(r^*, \Sigma)$, where $\Sigma$ is a diagonal matrix with entries $\sigma^2(x, y)$. Let $\pi_0$ be a reference policy (e.g., from pretraining), and define $d = \pi - \pi_0$ as the pointwise difference between the current and reference policies. Then, the inner product $\hat{R}^\top d$ is normally distributed with mean $r^{*\top} d$ and variance $d^\top \Sigma d$.
\end{restatable}

\paragraph{Lower Bound on the True Objective Function:}
The following theorem provides a high-probability lower bound on the optimization objective that incorporates the uncertainty in the reward estimates. The proof is provided in Appendix~\ref{sec:Proofs}.

\begin{restatable}{theorem}{surrogate}
\label{thm:surrogate}
Under the reward observation model in Assumption~\ref{assm:gaussian}, for any $\beta > 0$, the following inequality holds with probability at least $1 - \exp\left(- \frac{\abs{\cX}\abs{\cY}}{\beta^2} \right)$:
\[
\sup_{d \in \left[-2, 2\right]^{\abs{\mathcal{X}}\abs{\mathcal{Y}}}} \left( \hat{R}^\top d - \beta \|d\|_{\Sigma} \right) 
\leq
\sup_{d \in \left[-2, 2\right]^{\abs{\mathcal{X}}\abs{\mathcal{Y}}}} r^{*\top} d.
\]
\end{restatable}
The above result implies that the optimization problem on the left-hand side serves as a high-probability lower bound for the true reward maximization problem, which depends on the unobservable function $r^*$. Since only noisy estimates $r^*$ are available through $\hat R$, we propose the following surrogate optimization objective:
\begin{align}
\label{eq:obj_fn}
    \sup_{d \in \left[-2, 2\right]^{\abs{\mathcal{X}}\abs{\mathcal{Y}}}} \hat{R}^\top d - \beta \|d\|_{\Sigma}.
\end{align}
This leads to the following constrained optimization problem:
\[
\max_\pi \hat{R}^{\top} \pi \quad \text{subject to} \quad (\pi - \pi_0)^\top \Sigma (\pi - \pi_0) \leq \epsilon,
\]
for some $\epsilon > 0$. The weighted quadratic constraint penalizes deviations from the reference policy more heavily for prompt–response pairs with higher uncertainty in the reward model, thereby incorporating variance-awareness into the optimization.

\begin{remark}
Variants of the constrained optimization problem above have appeared in prior literature. For example, the unconstrained version reduces to the vanilla policy gradient method~\citep{sutton1999policy}. In RLHF, the PPO algorithm~\citep{schulman2017proximal} imposes a KL-divergence constraint to regularize updates, though the choice of distance metric is flexible~\citep{ziegler2019fine}. An $\ell_2$ approximation to the KL constraint leads to an unweighted constraint of the form $\|\pi - \pi_0\|_2^2 \leq \epsilon$. Similarly, the Trust Region Policy Optimization (TRPO) algorithm~\citep{schulman2015trust} uses a Fisher-weighted trust region constraint derived from the natural policy gradient.
\end{remark}

\paragraph{Theoretical Analysis:} 
We compare the performance of the variance-aware LLM alignment methodology with its variance-unaware counterpart to evaluate how incorporating reward estimate uncertainty affects policy robustness and effectiveness—particularly in scenarios with noisy reward signals. We consider two policies, $\pi_1$ and $\pi_2$, derived from different constrained optimization formulations.

Let $\pi_1$ denote the ``vanilla'' policy obtained by solving the unweighted $\ell_2$-constrained problem:
\begin{equation}
\label{eq:pi1}
    \pi_1 = \arg\max_\pi \, \pi^\top \hat{R} \quad \text{subject to} \quad \|\pi - \pi_0\|_2^2 \leq \epsilon.
\end{equation}

Let $\pi_2$ denote a variance-aware policy that is obtained by solving the covariance-weighted $\ell_2$-constrained problem:
\begin{equation}
\label{eq:pi2}
    \pi_2 = \arg\max_\pi \, \pi^\top \hat{R} \quad \text{subject to} \quad \|\pi - \pi_0\|_\Sigma^2 \leq \tilde{\epsilon}.
\end{equation}

To ensure a fair comparison between the policies resulting from the two policy optimization formulations, we set $\tilde{\epsilon} = \lambda_{\min}(\Sigma) \cdot \epsilon$. This scaling ensures that the largest axis of the ellipsoid in the variance-aware constraint \eqref{eq:pi2} is aligned with the radius of the $\ell_2$ ball in the vanilla constraint \eqref{eq:pi1}. We evaluate the expected true rewards $\pi^\top r^*$ for both methods and compare them to the reference policy reward $\pi_0^\top r^*$. We show that the variance-aware policy is less likely to underperform relative to the reference policy than the vanilla policy, thereby demonstrating its robustness under uncertainty.

\begin{restatable}{theorem}{risk}
\label{thm:main}
Consider the vanilla policy $\pi_1$ and the variance-aware policy $\pi_2$ as defined in Equations~\eqref{eq:pi1} and~\eqref{eq:pi2}, respectively. Let $\tilde{\epsilon} = \lambda_{\min}(\Sigma) \cdot \epsilon$ so that the feasible set of the variance-aware optimization is no larger than that of the variance-unaware one. Then, the following inequality holds:
\[
\mathbb{P}\left( \pi_2^\top r^* \leq \pi_0^\top r^* \right) \leq \mathbb{P}\left( \pi_1^\top r^* \leq \pi_0^\top r^* \right).
\]
\end{restatable}

\begin{remark}
Theorem~\ref{thm:main} highlights a risk–reward trade-off. While the vanilla policy $\pi_1$ may yield higher rewards when $\hat{R}$ closely approximates $r^*$, it is riskier because it disregards the uncertainty in reward estimates. By contrast, the variance-aware policy $\pi_2$ explicitly accounts for variance and reduces the probability of underperforming relative to the reference policy $\pi_0$. The proof is provided in Appendix~\ref{sec:Proofs}.
\end{remark}

\begin{remark}
The variance-aware policy is closely related to the notion of the reward-to-variability ratio, known in finance as the Sharpe Ratio~\citep{sharpe1966mutual}, which balances expected return against risk.

\begin{restatable}{theorem}{sharpe}
Consider the following constrained optimization problem:
\[
\max_\pi \; \mathbb{E}_{x \sim \rho_{\cX},\, y \sim \pi(\cdot|x)} \left[ \hat{R}(x, y) \right] \quad \text{subject to} \quad \mathbb{E}_{x \sim \rho_{\cX},\, y \sim \pi(\cdot|x)} \left[ \sigma^2(x, y) \ln \frac{\pi(y|x)}{\pi_0(y|x)} \right] \leq \epsilon,
\]
where $\hat{R}(x, y)$ and $\sigma^2(x, y)$ are the estimated reward and its variance. The optimal policy has the form:
\[
\pi^*(y|x) \propto \pi_0(y|x) \exp\left( \frac{\hat{R}(x, y)}{\beta \sigma^2(x, y)} \right),
\]
for some $\beta > 0$. Thus, the optimal policy assigns higher probability to actions with a high reward-to-variance ratio, analogous to the Sharpe Ratio in portfolio theory.
\end{restatable}
\end{remark}

\begin{figure}[!htbp]
    \centering
    \begin{minipage}{0.48\linewidth}
        \includegraphics[width=\linewidth]{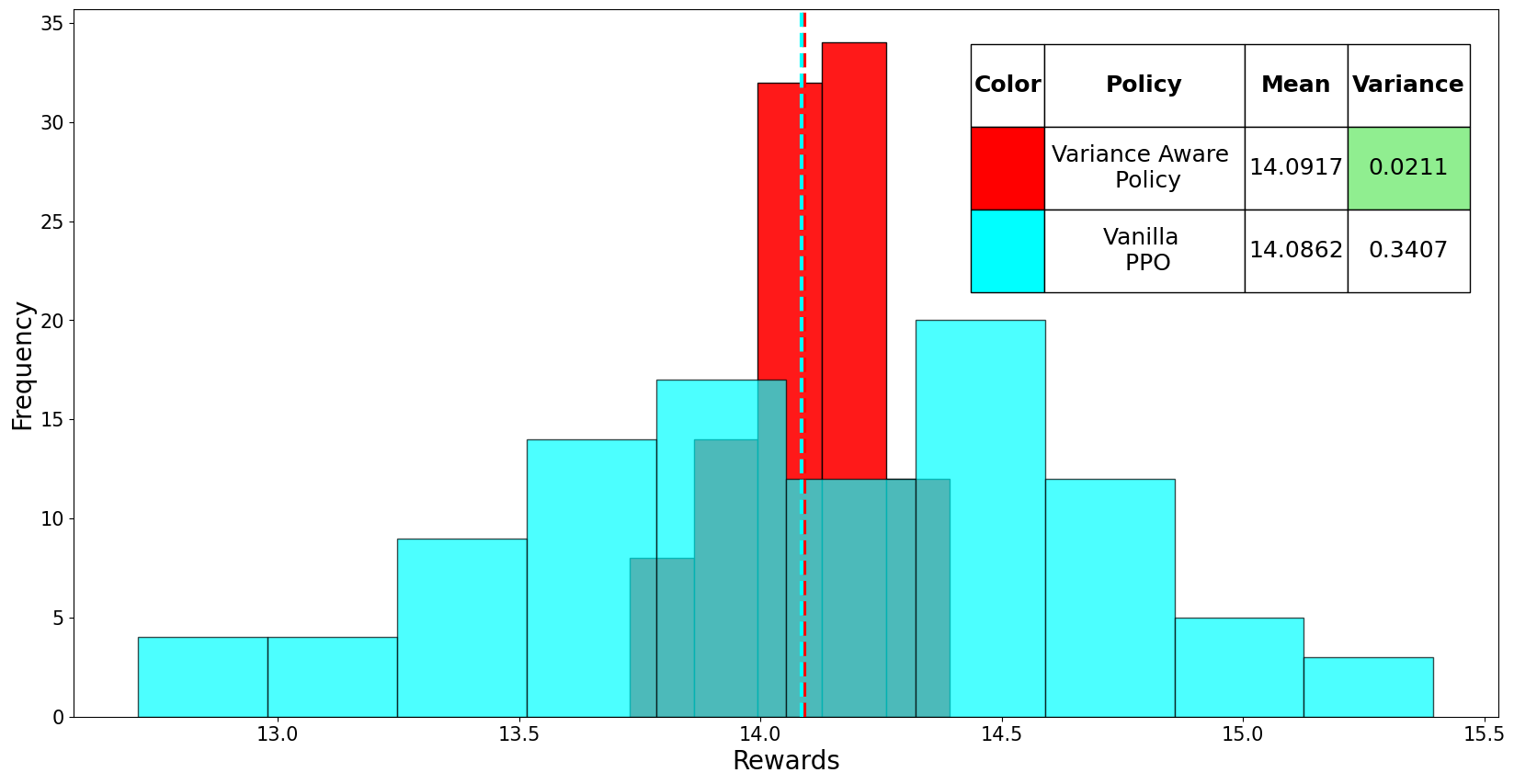}
        \caption*{\scriptsize High-variability setting: Reward variances sampled from $(3, 100)$. The variance-aware policy yields a significantly sharper distribution than the vanilla policy. Variance ratio: $\sigma^2_{\text{vanilla}} / \sigma^2_{\text{variance-aware}} = 16.12$.}
    \end{minipage}
    \hfill
    \begin{minipage}{0.48\linewidth}
        \includegraphics[width=\linewidth]{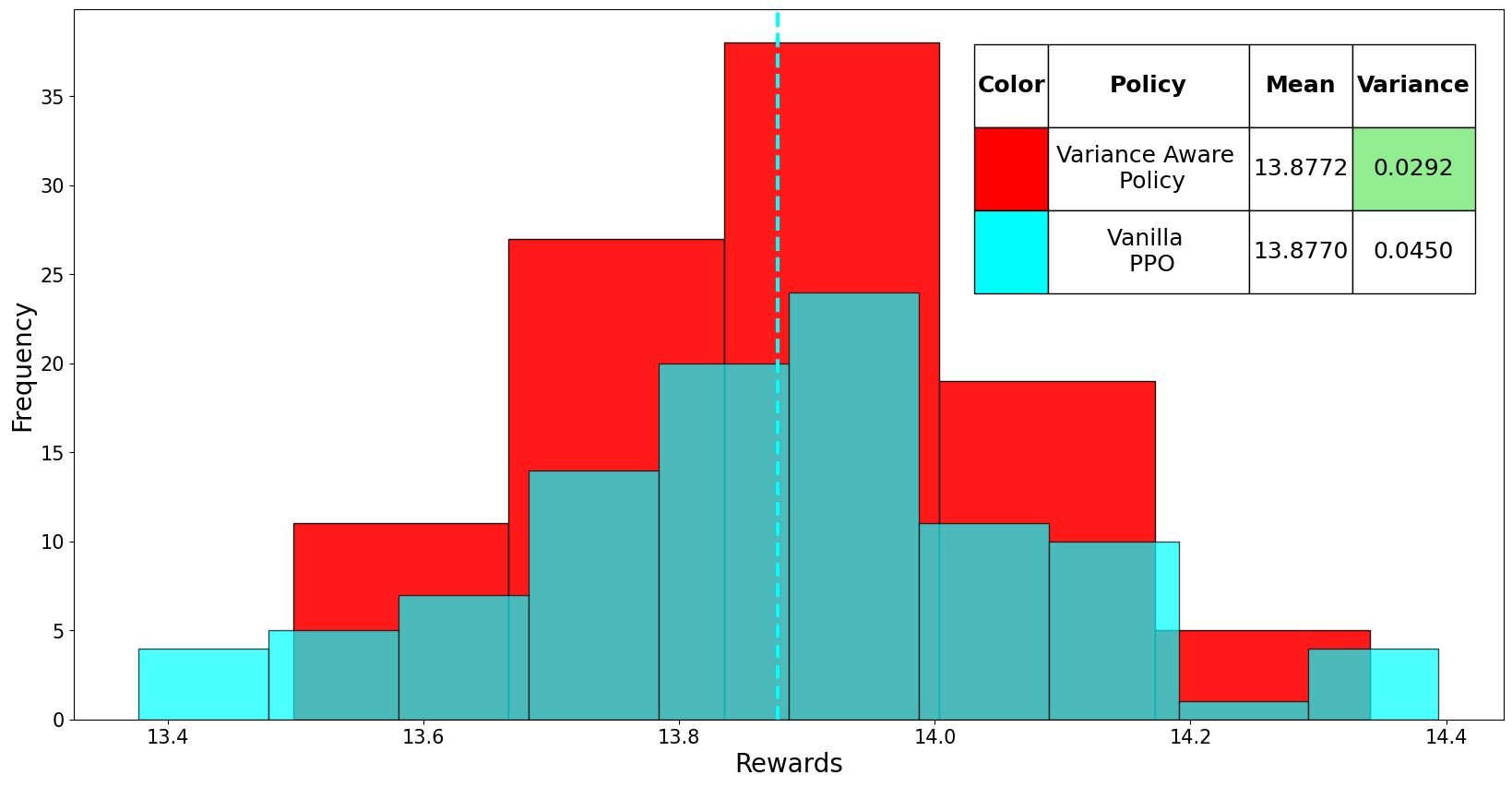}
        \caption*{\scriptsize Low-variability setting: Reward variances sampled from $(3, 10)$. The variance-aware policy still exhibits lower variance, though the improvement is less pronounced. Variance ratio: $\sigma^2_{\text{vanilla}} / \sigma^2_{\text{variance-aware}} = 1.54$.}
    \end{minipage}
    \caption{\footnotesize Distribution of policy returns under different reward variance settings. In both cases, the true reward vector $r^*$ is fixed, and reward estimates $\hat{R}$ are sampled from a multivariate Gaussian with varying covariance matrices. While the variance-aware policy consistently exhibits lower return variance, the relative gain is substantially higher in the high-variability setting, highlighting its robustness under uncertainty.}
    \label{fig:distribution}
\end{figure}

The advantages of the variance-aware method become more pronounced when there is greater heterogeneity in reward estimate variances across prompt–response pairs. In Figure~\ref{fig:distribution}, we simulate two settings: one where the variances are drawn from the interval $(3, 100)$ and another from $(3, 10)$. As evident from the figure, the variance-aware method produces a notably sharper distribution in the high-variability setting, underscoring its effectiveness in managing reward uncertainty.

%% file: Neurips/Experiments.tex
\section{Experiments}
\label{sec:experiment}
In this section, we corroborate our findings through empirical studies based on modern LLMs. We leverage the uncertainty in reward scores to fine-tune large language models (LLMs) using a novel, variance-aware variant of the  Proximal Policy Optimization (PPO) strategy (Algorithm \ref{alg:unified-ppo}). We compare the reward distributions obtained from the variance-aware policies with those from a vanilla PPO baseline.
Our results indicate that variance-aware policies generate content whose reward distribution exhibits significantly lower variance than that of the vanilla baseline. 

\subsection{Uncertainty Modeling}
A key challenge in this pipeline is generating meaningful uncertainty estimates for reward scores. We explore two methodologies to address this:
\begin{enumerate}[nosep,leftmargin=*,label=--]
    \item \textit{Ensemble Reward Models (Details in Appendix~\ref{sec:reward_model}):} We fine-tune an ensemble of 10 \textsc{Gemma-2B-it}\cite{gemma-2b-it} based reward models on a preference dataset. The noisy reward $\hat{R}$ is the ensemble mean, and uncertainty is measured via sample variance. For evaluation, we use \textsc{FsfairX-LLaMA3-RM-v0.1} as the true reward model $r^\ast$\citep{dong2023raft, xiong2024iterative}.\footnote{\href{https://huggingface.co/sfairXC/FsfairX-LLaMA3-RM-v0.1}{https://huggingface.co/sfairXC/FsfairX-LLaMA3-RM-v0.1}}
    \item \textit{Prompted Reward Models: } To simulate the scenario where a diverse population provides direct reward scores, we use large language models (\textsc{Gemini-1.5-flash}\cite{team2024gemini}, \textsc{Gemini-2.0-flash}\cite{google2024gemini2flash}, and \textsc{Deepseek-V3}\cite{DeepSeekAI2024DeepSeekV3TR}) as prompted reward models (abridged prompt appears below; full prompt is in Appendix~\ref{app:prompt}). Each model returns a confidence interval $[a, b]$ under the constraint $1 \leq a < b \leq R$, with smaller intervals for content whose interpretation is clear and wider ones for ambiguity. We assume the noisy reward $\hat{R}$ is uniformly distributed over $[a, b]$, with variance $(b-a)^2/12$, and define the true reward as $r^\ast = (a+b)/2$. This score interval is to represent the variability in human bias while scoring.
\end{enumerate}
\begin{center}
\begin{tcolorbox}[
  enhanced,
  colback=lightgray,
  colframe=lightgray,
  boxrule=0pt,
  sharp corners,
  width=0.9\textwidth,
  ]
  \label{box:prompt}
  \scriptsize{ % Font size reduced
  You represent a \textit{diverse} population of human beings. Your task is to evaluate the aesthetic value of a given text. Note that aesthetic value is subjective and can vary widely across individuals and cultures.
  Return two real numbers `a' and `b' such that:\\
  - $1 \leq a < b \leq 100$\\
  - If the population is likely to disagree strongly, widen the range else, narrow the range}
\end{tcolorbox}
\end{center}

\subsection{Variance Aware PPO}

\begin{minipage}[!htbp]{0.48\textwidth}
\centering
\begin{algorithm}[H]
\caption{PPO: Vanilla / Variance-Aware}
\begin{algorithmic}[1]
\footnotesize
\State Initialize policy $\pi_{\theta}$ with reference policy $\pi_0$
\State Choose method: \textbf{Vanilla PPO} or \textbf{Variance-Aware PPO}
\For{iteration = 1 to $N$}
    \State Sample batch $\mathbf{x} = \{x_i\}_{i=1}^B$
    \State Generate $\mathbf{y} \sim \pi_{\theta}(\cdot|\mathbf{x})$
    \State Compute rewards $\mathbf{r} = \hat{R}(\mathbf{x}, \mathbf{y})$
    \If{Variance-Aware}
        \State Also compute variances $\boldsymbol{\sigma}^2$
        \State Compute objective:
        \[
        \mathcal{L}(\theta) = \sum_{i=1}^{B}  r_i - \sigma_i^2 \log \frac{\pi_{\theta}(y_i|x_i)}{\pi_0(y_i|x_i)}
        \]
    \Else
        \State Compute objective:
        \[
        \mathcal{L}(\theta) = \sum_{x,y,r} r(x, y) - \log \frac{\pi_{\theta}(y|x)}{\pi_0(y|x)}
        \]
    \EndIf
    \State Update $\theta$ via gradient ascent on $\mathcal{L}(\theta)$
\EndFor
\end{algorithmic}
\label{alg:unified-ppo}
\end{algorithm}
\end{minipage}
\hfill
\begin{minipage}[!htbp]{0.5\linewidth}
\begin{algorithm}[H]
\caption{Evaluation of $\pi^\top r^\ast$}
\begin{algorithmic}[1]
\centering
\footnotesize
\State Load policy $\pi$
\State Sample prompts $\{x_k\}_{k=1}^M$
\State Generate responses $\{y_k \sim \pi(\cdot|x_k)\}_{k=1}^M$
\State Compute rewards $\{r^*(x_k, y_k)\}_{k=1}^M$
\State Estimate mean reward: 
\[
    \hat{r} = \frac{1}{M} \sum_{k=1}^{M} r^*(x_k, y_k)
\]
\State \textbf{Return:} $\hat{r}$
\end{algorithmic}
\label{alg:eval}
\end{algorithm} 
\end{minipage}

We fine-tune \textsc{GPT-2}\cite{radford2019language}, \textsc{Qwen2.5-0.5B}\cite{Yang2024Qwen25TR}, and \textsc{Mistral-7B}\cite{jiang2023mistral} using reward uncertainty. \textsc{GPT-2} is paired with an ensemble reward model, while \textsc{Qwen2.5-0.5B} and \textsc{Mistral-7B} are fine-tuned with prompted reward models derived from \textsc{Gemini-1.5-flash}, \textsc{Gemini-2.0-flash}, and \textsc{Deepseek-V3}. For each model–reward model pair, we independently train eighty vanilla PPO policies and eighty variance-aware policies (Algorithm~\ref{alg:unified-ppo}), resulting in $\{\pi_1^i\}_{i=1}^{80}$ and $\{\pi_2^i\}_{i=1}^{80}$. We then estimate the reward distributions under the true reward model for three policies: (i) the reference policy, (ii) vanilla PPO fine-tuned policies, and (iii) variance-aware fine-tuned policies. Using Algorithm~\ref{alg:eval}, we compute empirical rewards ${\langle \pi_1^j, r^\ast \rangle}_{j=1}^{80}$ and ${\langle \pi_2^j, r^\ast \rangle}_{j=1}^{80}$, each consisting of 80 samples. We similarly resample from the reference policy $\pi_0$ to construct its empirical reward distribution. Figures~\ref{fig:GPT-2} and~\ref{fig:llm-reward} visualize the resulting distributions across different LLM–reward model pairs. Full details of algorithms, experiment hyperparameters, and hardware details are deferred to Appendix~\ref{sec:PPO}.

\begin{figure}[!htbp]
\centering
\includegraphics[width=0.7\linewidth]{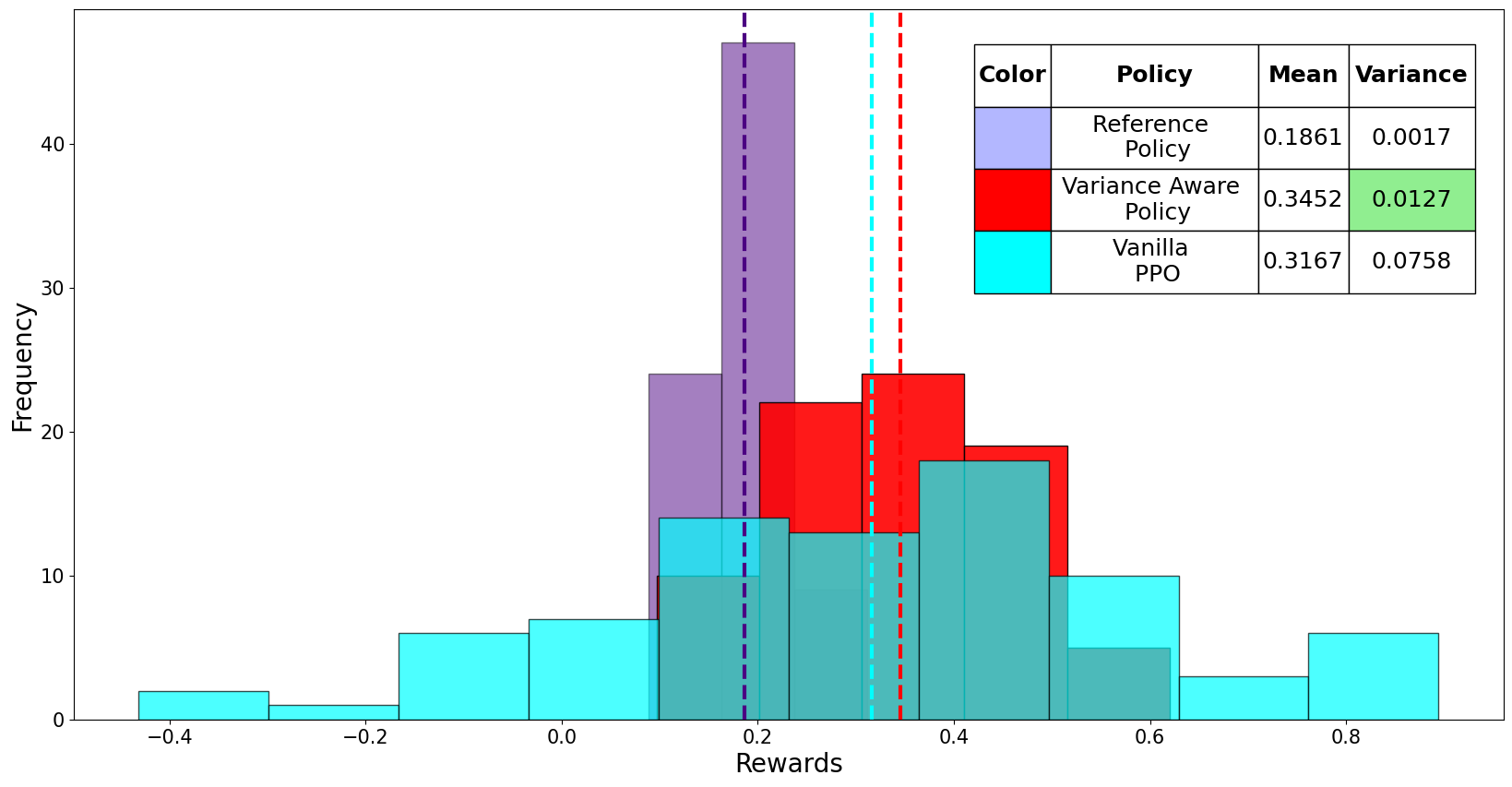}
    \captionof{figure}{\footnotesize{True empirical reward distributions for \textsc{GPT-2} fine-tuned with a custom ensemble. Shown are policies trained with the vanilla method, variance-aware method, and the reference policy. The true reward $r^\ast$ is based on a larger \textsc{Llama-3-8B} model.}}
    \label{fig:GPT-2}
\end{figure}

\paragraph{Observations: } Figure~\ref{fig:GPT-2} represents the empirical reward distributions for \textsc{GPT-2} with an ensemble reward model. The sub-figures in Figure \ref{fig:llm-reward} represent the empirical reward distributions for \textsc{Mistral-7B} and \textsc{Qwen2.5-0.5B} with the prompted reward models of \textsc{Gemini-1.5-flash, Gemini-2-flash} and \textsc{Deepseek-V3}. We observe that the variance-aware policies consistently enjoy lower variances than the vanilla policies. We summarize the results in the Table \ref{tab:reward_variance}. 
\begin{table}[!htbp]
    \centering
    \footnotesize
\begin{tabular}{|c|c|c|c|c|}
    \hline
    \textbf{Model} & \textbf{Reward Model} & \makecell{\textbf{Vanilla Policy} \\ $\sigma_1^2$} & \makecell{\textbf{Variance-Aware} \\ \textbf{Method} \\ $\sigma_2^2$} & \textbf{F-score $\frac{\sigma_1^2}{\sigma_2^2}$} \\
    \hline
    \textsc{Gpt-2} & \textsc{Ensemble Rm} & 0.076 & 0.012 & 6.33 \\
    \hline
    \multirow{3}{*}{\textsc{Mistral-7b}} 
        & \textsc{Gemini-1.5-flash} & 0.13 & 0.04 & 3.25 \\
        & \textsc{Gemini-2.0-flash} & 0.18 & 0.02 & 9.00 \\
        & \textsc{Deepseek-v3} & 0.06 & 0.02 & 3.00 \\
    \hline
    \multirow{3}{*}{\textsc{Qwen2.5-0.5b}} 
        & \textsc{Gemini-1.5-flash} & 0.05 & 0.02 & 2.50 \\
        & \textsc{Gemini-2.0-flash} & 0.09 & 0.02 & 4.50 \\
        & \textsc{Deepseek-v3} & 0.05 & 0.02 & 2.50 \\
    \hline
\end{tabular}
    \caption{\footnotesize{The table reports the variance of the reward distribution for different LLM–Reward Model pairs under the vanilla and variance-aware fine-tuning methods. It can be observed that the variance-aware method consistently yields statistically lower variance compared to the vanilla method. This is confirmed by the reported $F$-scores, all of which exceed the critical value $F_{0.975} = 1.616$ at the $95\%$ confidence level. Thus, we reject the null hypothesis of equal variances and conclude that the variance reduction achieved by the variance-aware method is statistically significant.}}
    \label{tab:reward_variance}
\end{table}
For samples of size 80, the critical $F$-value at the $95\%$ confidence level is approximately $1.616$. As shown in the last column of Table~\ref{tab:reward_variance}, the variance ratios for all model-reward model pairs exceed this threshold which indicates that the differences in variance between the vanilla and variance-aware methods are statistically significant. We note from Table \ref{tab:reward_means} that the means of the two methods are statistically identical at the $95\%$ confidence interval. 

\begin{table}[!htbp]
\begin{center}
    \centering
\hspace*{-1.5em}
    \footnotesize
    \begin{tabular}{|c|c|c|c|c|c|}
        \hline
        \textbf{Model} & \textbf{Reward Model} & \makecell{\textbf{Vanilla Policy} \\ $\mu_1$} & \makecell{\textbf{Variance-Aware} \\ \textbf{Method} \\ $\mu_2$} & \makecell{\textbf{t-score} \\ $\left|\frac{\mu_1 - \mu_2}{\sqrt{\sigma_1^2/n + \sigma_2^2/n}}\right|$} & $t_{crit} = t_{0.975}$\\
        \hline
        \textsc{Gpt-2} & \textsc{Ensemble Rm} & 0.31 & 0.34 & 0.86 & 1.98\\
        \hline
        \multirow{3}{*}{\textsc{Mistral-7b}} 
            & \textsc{Gemini-1.5-flash} & 29.21 & 29.25 & 0.92 & 1.98\\
            & \textsc{Gemini-2.0-flash} & 33.09 & 33.06 & 0.71 & 1.99\\
            & \textsc{Deepseek-V3} & 45.97 & 45.98 & 0.39 & 1.98 \\
        \hline
        \multirow{3}{*}{\textsc{Qwen2.5-0.5b}} 
            & \textsc{Gemini-1.5-flash} & 28.37 & 28.36 & 0.13 & 1.98\\
            & \textsc{Gemini-2.0-flash} & 32.12 & 32.13 & 0.42 & 2.00\\
            & \textsc{Deepseek-V3} & 46.01 & 46.00 & 0.28 & 1.98\\
        \hline
\end{tabular}
    \caption{\footnotesize{The table reports the means of the reward distributions for different LLM–Reward Model pairs under the vanilla and variance-aware fine-tuning methods. To assess whether the differences in means are statistically significant, we conduct Welch's $t$-test \cite{welch_t_1947}. The resulting $t$-scores and the critical $t$-value at the $95\%$ confidence level are shown in Columns~5 and~6, respectively. As observed, all computed $t$-scores fall well below the critical value, indicating that we fail to reject the null hypothesis. This suggests that the means of the reward distributions under the two fine-tuning methods are not statistically different.}}
    \label{tab:reward_means}
\end{center}
\end{table}

We define the empirical risk as the probability that a policy underperforms the reference policy, i.e., $\mathbb{P}(\hat{R} \leq \pi_0^\top r^\ast)$. As shown in Table~\ref{tab:risk_profile}, the estimated risk for the variance-aware policies is statistically lower than the vanilla polies at the 95\% confidence level~\citep{montgomery2014applied}. \begin{table}[!htbp]
    \centering
    \footnotesize
    \begin{tabular}{|c|c|c|c|c|}
    \hline
    \textbf{Model} & \textbf{Reward Model} & \makecell{\textbf{Variance Aware} \\ \textbf{Risk} $\hat p_2$} & \makecell{\textbf{Vanilla Policy} \\ \textbf{Risk} $\hat p_1$ } & $p_{\text{value}}$ \\
    \hline
    \textsc{Gpt-2} & \textsc{Ensemble Rm} & 0.05 & 0.29 & $3\times10^{-5}$ \\
    \hline
    \multirow{3}{*}{\textsc{Mistral-7b}} 
        & \textsc{Gemini-1.5-flash} & 0.41 & 0.48 & $0.0282$ \\
        & \textsc{Gemini-2.0-flash} & 0.07 & 0.36 & $3\times10^{-6}$ \\
        & \textsc{Deepseek-v3} & 0.06 & 0.18 & $0.014$ \\
    \hline
    \multirow{3}{*}{\textsc{Qwen2.5-0.5b}} 
        & \textsc{Gemini-1.5-flash} & 0.24 & 0.29 & $0.038$ \\
        & \textsc{Gemini-2.0-flash} & 0.16 & 0.39 & $0.0007$ \\
        & \textsc{Deepseek-v3} &  0.24 & 0.29 & $0.042$ \\
    \hline
\end{tabular}
    \caption{\footnotesize{The table reports the empirical risk for different LLM–Reward Model pairs under the vanilla and variance-aware fine-tuning methods. The variance-aware method consistently yields statistically lower risk than the vanilla method. This is confirmed by a one-sided two-proportion Z-test. The reported $p$-values are below the critical threshold $\alpha = 0.05$, allowing us to reject the null hypothesis that $p_2 \geq p_1$ in favor of the alternative $p_2 < p_1$ at the 95\% confidence level.}}
    \label{tab:risk_profile}
\end{table}

\begin{figure}[!htbp]
\centering
\begin{minipage}[t]{0.5\linewidth}
    \includegraphics[width=\linewidth]{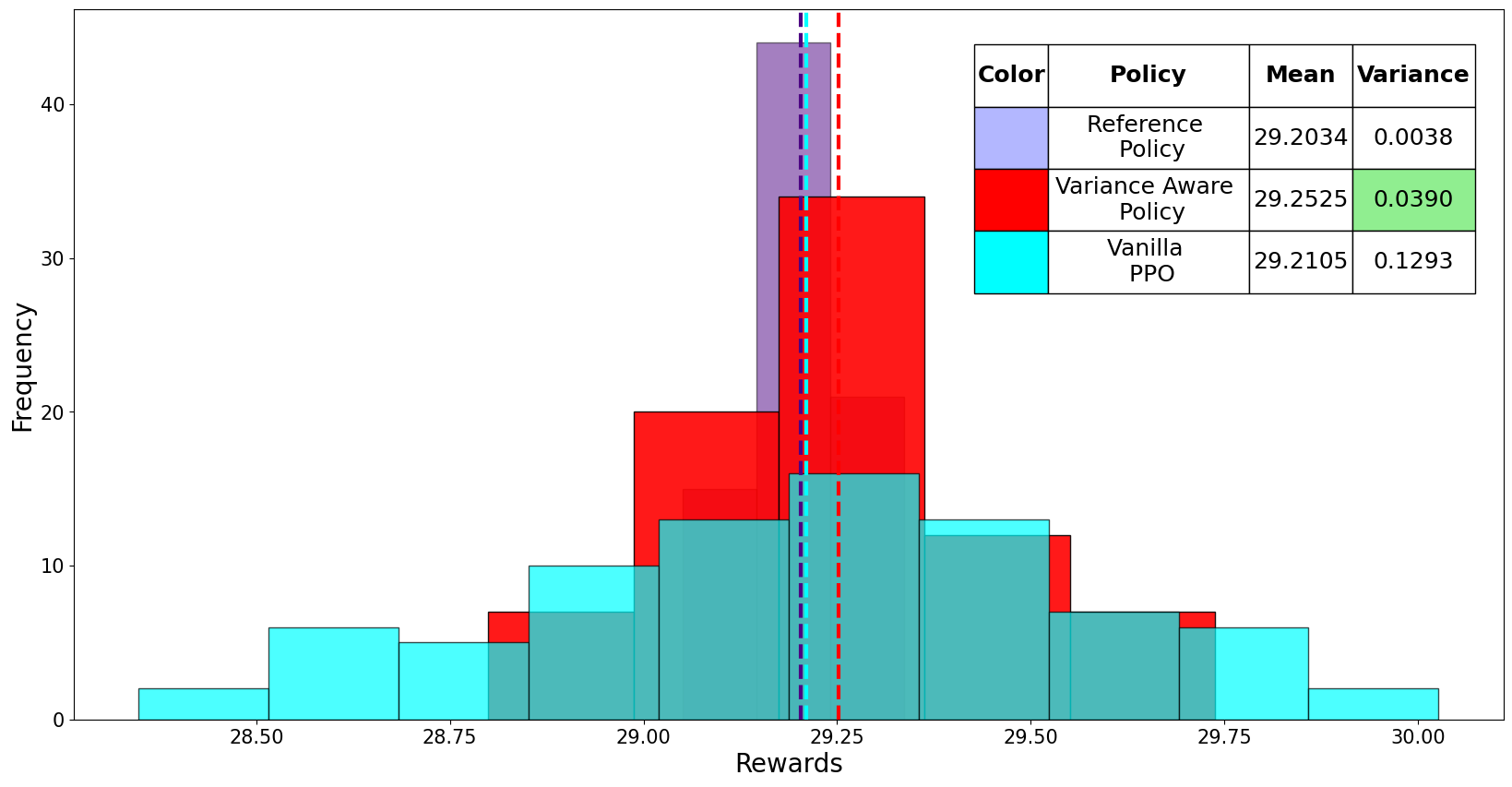}
    \caption*{\scriptsize (a) \textsc{Mistral-7b} finetuned with \textsc{Gemini-1.5-flash}}
\end{minipage}%
\hfill
\begin{minipage}[t]{0.5\linewidth}
    \includegraphics[width=\linewidth]{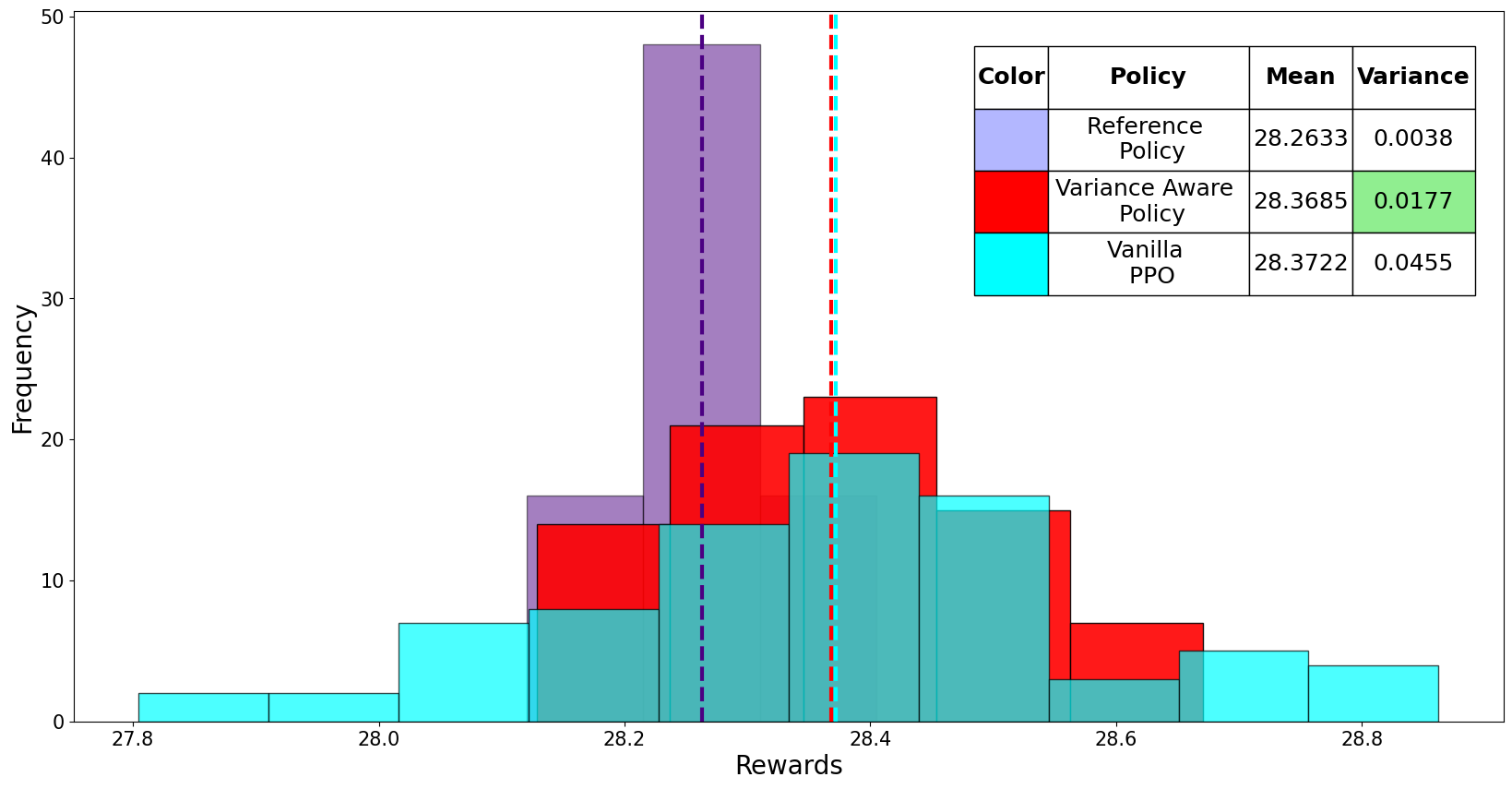}
    \caption*{\scriptsize (b) \textsc{Qwen2.5-0.5b} finetuned with \textsc{Gemini-1.5-flash}}
\end{minipage}
\begin{minipage}[t]{0.5\linewidth}
    \includegraphics[width=\linewidth]{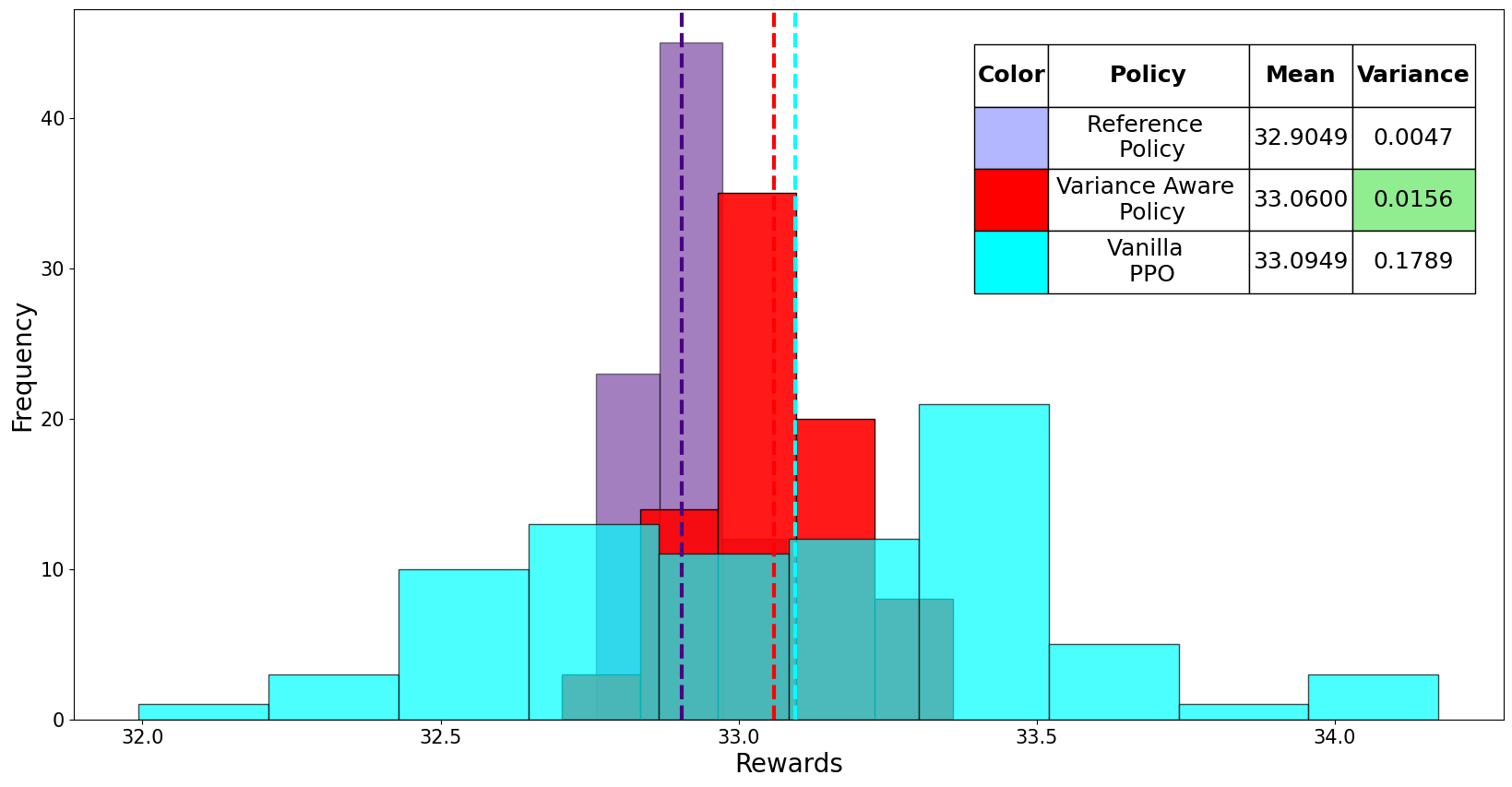}
    \caption*{\scriptsize (c) \textsc{Mistral-7b} finetuned with \textsc{Gemini-2.0-flash}}
\end{minipage}%
\hfill
\begin{minipage}[t]{0.5\linewidth}
    \includegraphics[width=\linewidth]{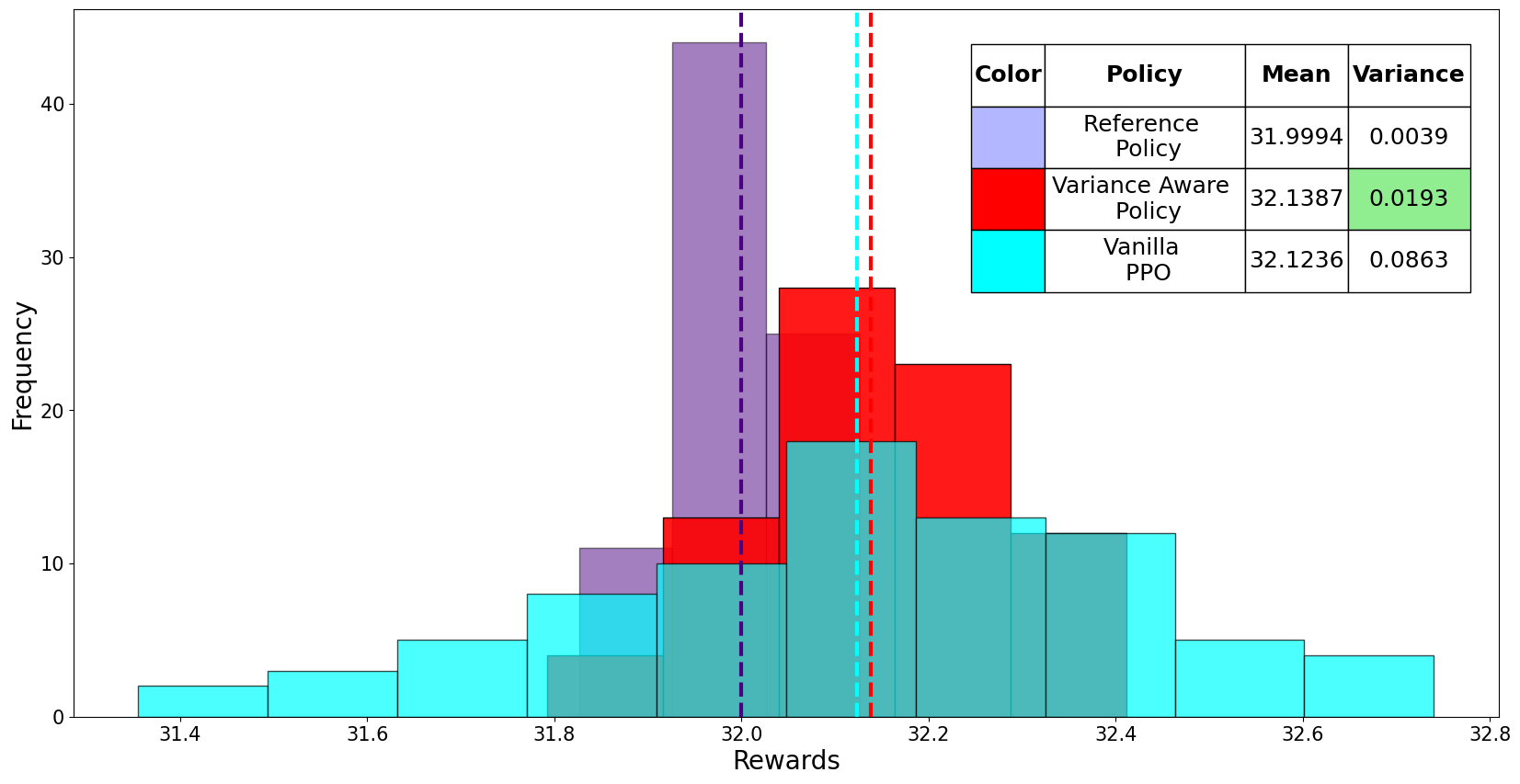}
    \caption*{\scriptsize (d) \textsc{Qwen2.5-0.5b} finetuned with \textsc{Gemini-2.0-flash}}
\end{minipage}
\begin{minipage}[t]{0.5\linewidth}
    \includegraphics[width=\linewidth]{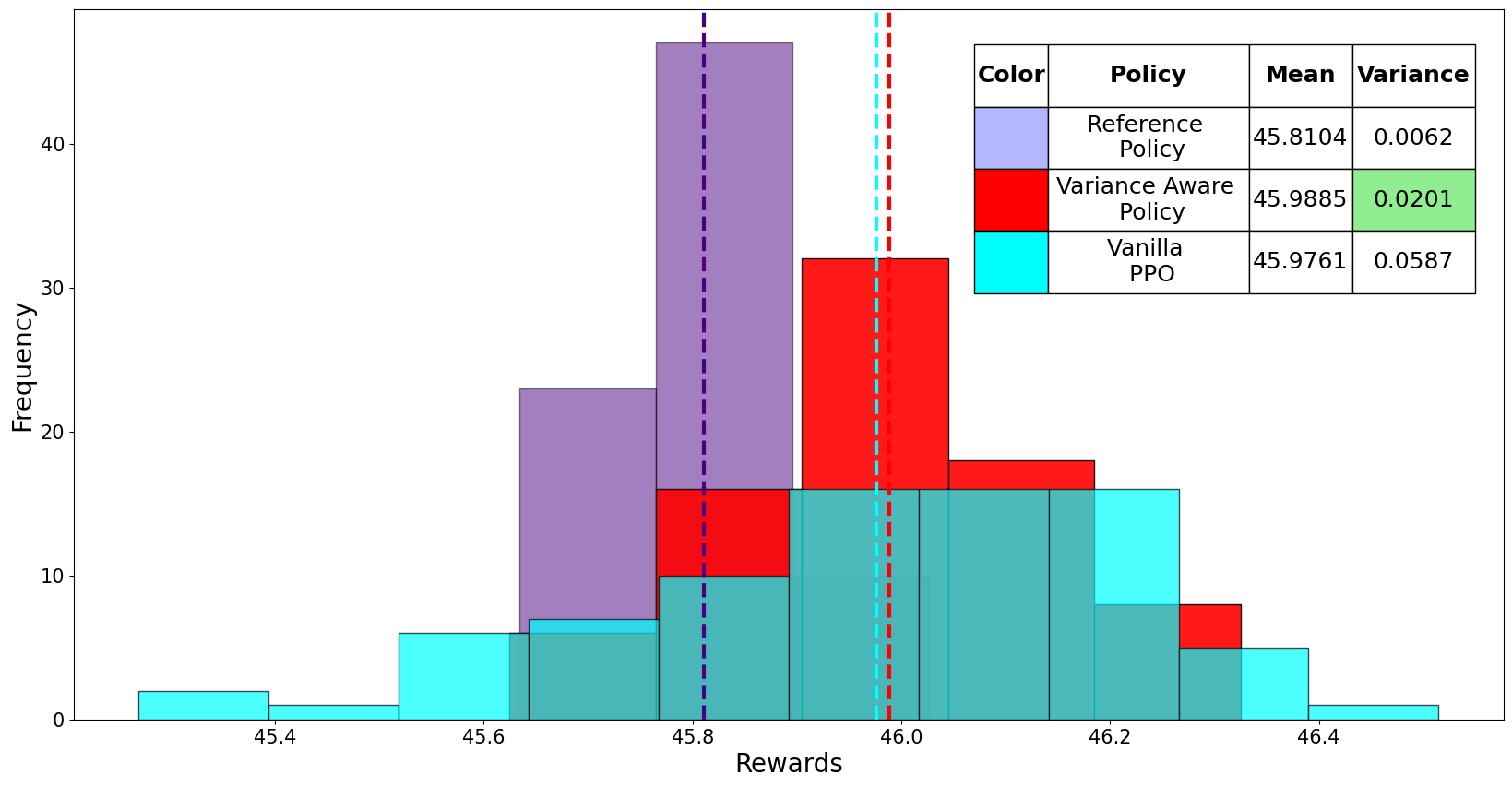}
    \caption*{\scriptsize (e) \textsc{Mistral-7b} finetuned with \textsc{Deepseek-v3}}
\end{minipage}%
\hfill
\begin{minipage}[t]{0.5\linewidth}
    \includegraphics[width=\linewidth]{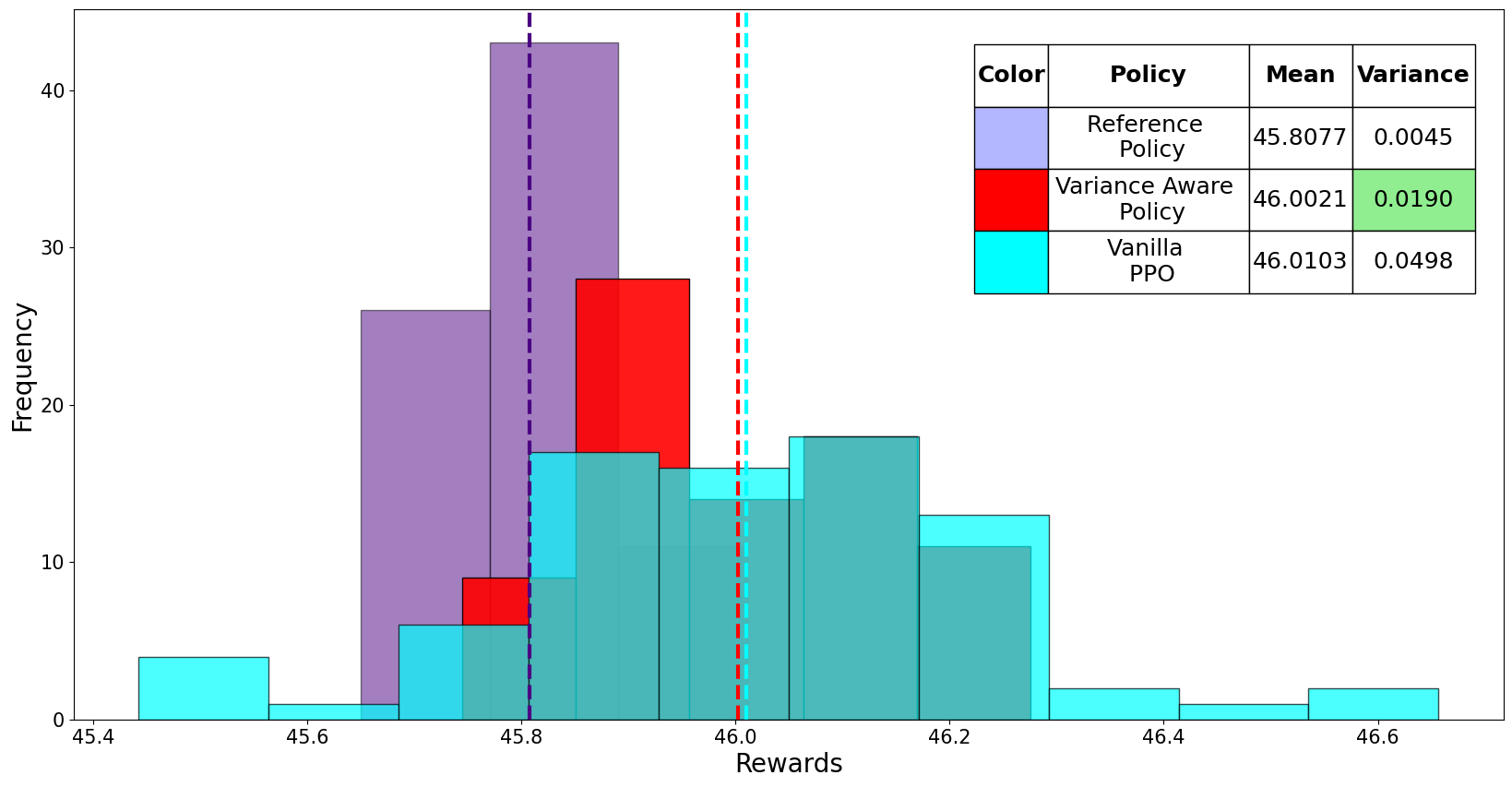}
    \caption*{\scriptsize (f) \textsc{Qwen2.5-0.5b} finetuned with \textsc{Deepseek-v3}}
\end{minipage}
\caption{\footnotesize
Empirical reward distributions for \textsc{Mistral} and \textsc{Qwen} models fine-tuned using prompted reward models: \textsc{Gemini-1.5-flash}, \textsc{Gemini-2.0-flash}, and \textsc{Deepseek-v3}. Results are shown for vanilla policies, variance-aware policies, and reference policies. The noisy reward used during fine-tuning, $\hat{R}$, is sampled uniformly from the interval $[a, b]$ with $1 \leq a < b \leq 100$, and the true reward $r^\ast$ used for evaluation is computed as the midpoint $(a + b)/2$.
}
\label{fig:llm-reward}
\end{figure}

\paragraph{Effect of Variability in Reward Model Variance:} 
We examine how the variability in reward model variance impacts the effectiveness of variance-aware fine-tuning. We fine-tune \textsc{Qwen2.5-0.5B} using the prompted \textsc{Gemini-2.0-flash} reward model, varying the reward range from $[1, 100]$ to $[1, 10]$. Narrower ranges lead to lower variability in the standard deviations $\sigma(x, y)$ across prompt–response pairs, as shown in Figure~\ref{fig:variability-of-variance}~(e). Fine-tuning with these reward distributions reveals that as variance variability decreases, the benefit of the variance-aware method diminishes—evidenced by a declining variance ratio compared to the vanilla method. Results are summarized in Table~\ref{tab:my_label}.

\begin{figure}[!htbp]
\centering
\begin{minipage}[t]{0.5\linewidth}
    \includegraphics[width=\linewidth]{Figures/Quen-Gemini-2.png}
    \captionof*{figure}{\scriptsize (a) Finetuned with rewards bounded between $1 \leq \hat{R} \leq 100$}
\end{minipage}%
\hfill
\begin{minipage}[t]{0.5\linewidth}
    \includegraphics[width=\linewidth]{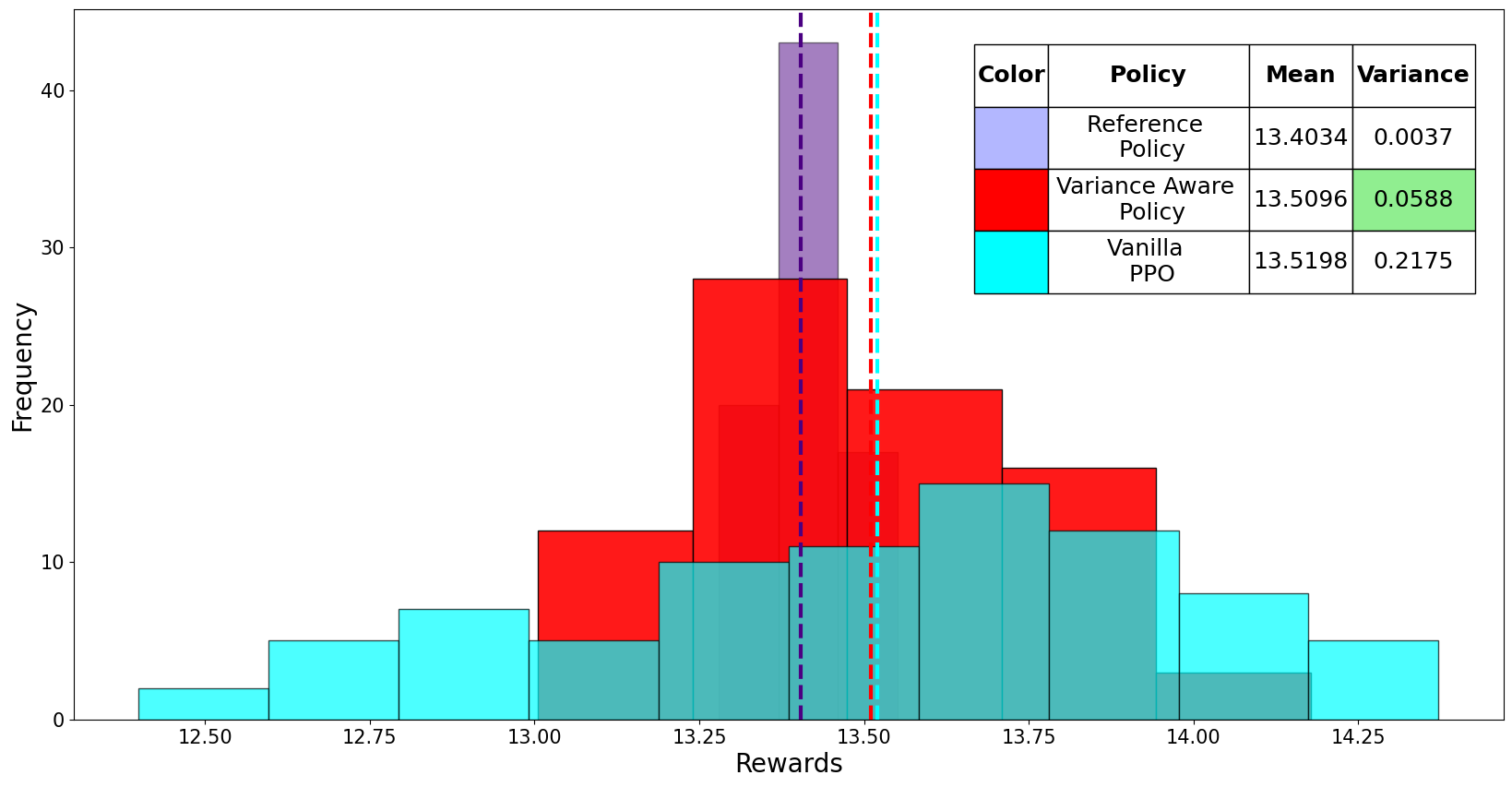}
    \caption*{\scriptsize (b) Finetuned with rewards bounded between $1 \leq \hat{R} \leq 50$}
\end{minipage}
\begin{minipage}[t]{0.5\linewidth}
    \includegraphics[width=\linewidth]{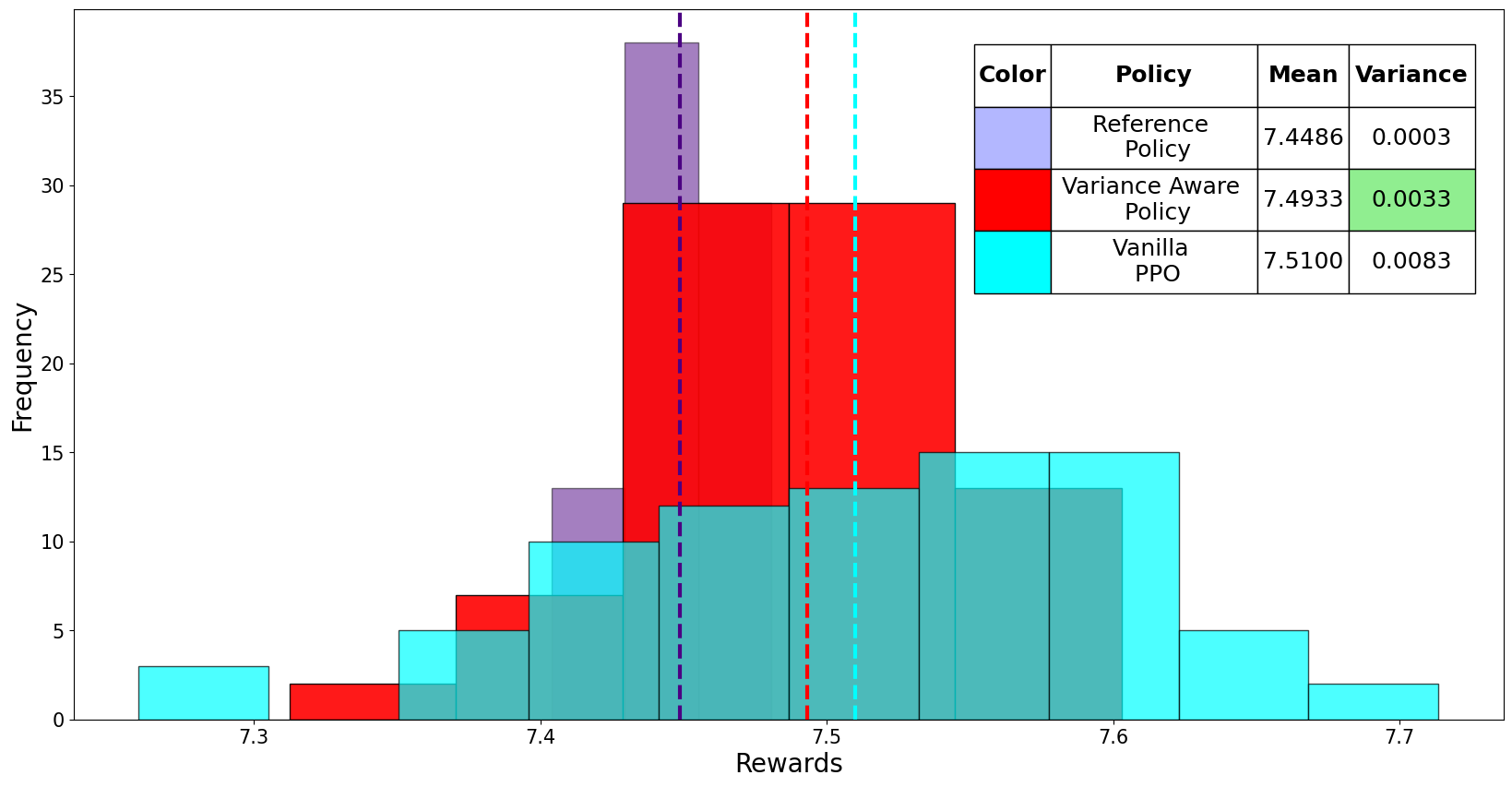}
    \caption*{\scriptsize (c) Finetuned with rewards bounded between $1 \leq \hat{R} \leq 20$}
\end{minipage}%
\hfill
\begin{minipage}[t]{0.5\linewidth}
    \includegraphics[width=\linewidth]{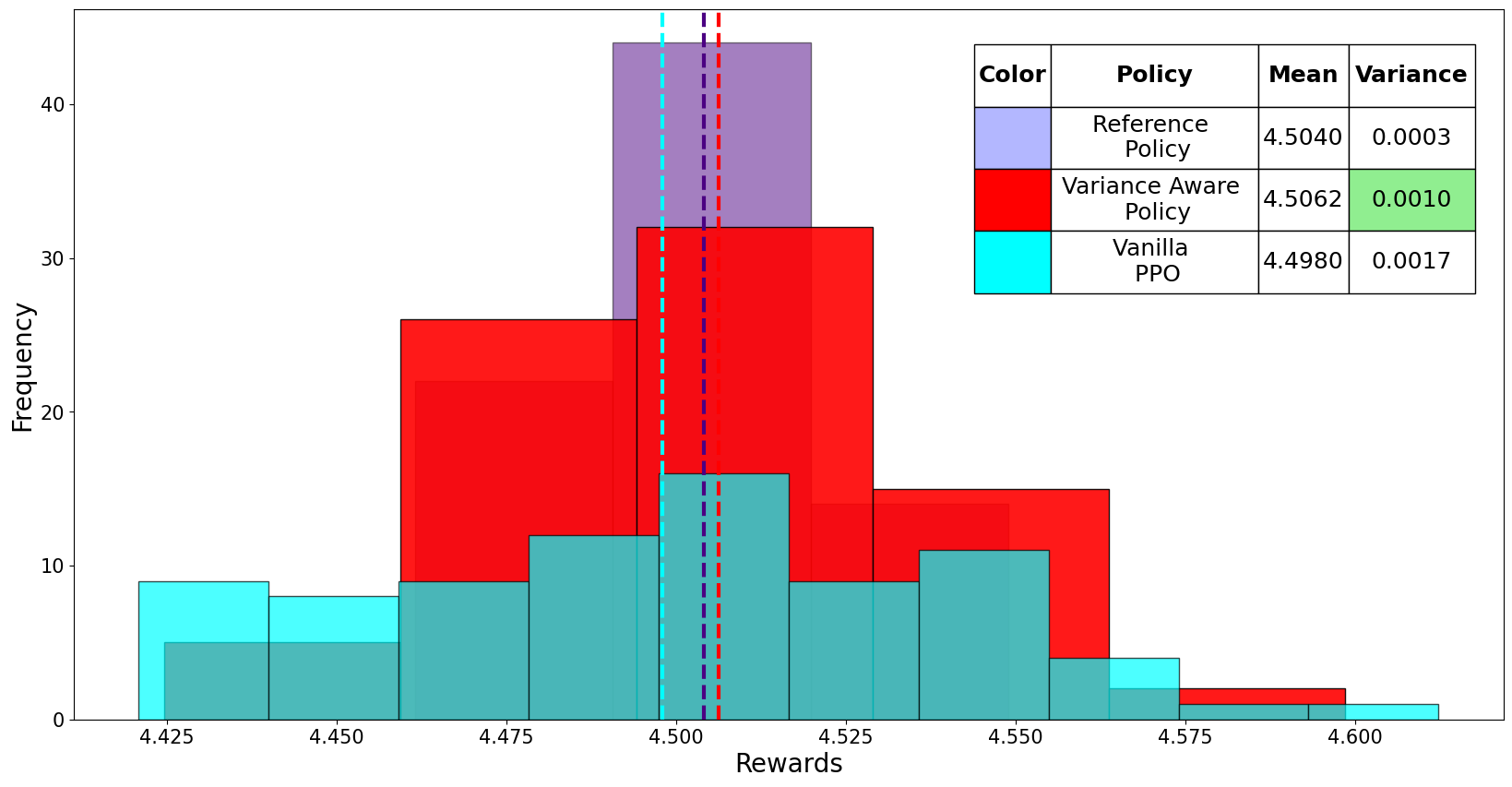}
    \caption*{\scriptsize (d) Finetuned with rewards bounded between $1 \leq \hat{R} \leq 10$}
\end{minipage}
\begin{minipage}[t]{0.70\linewidth}
    \includegraphics[width=\linewidth]{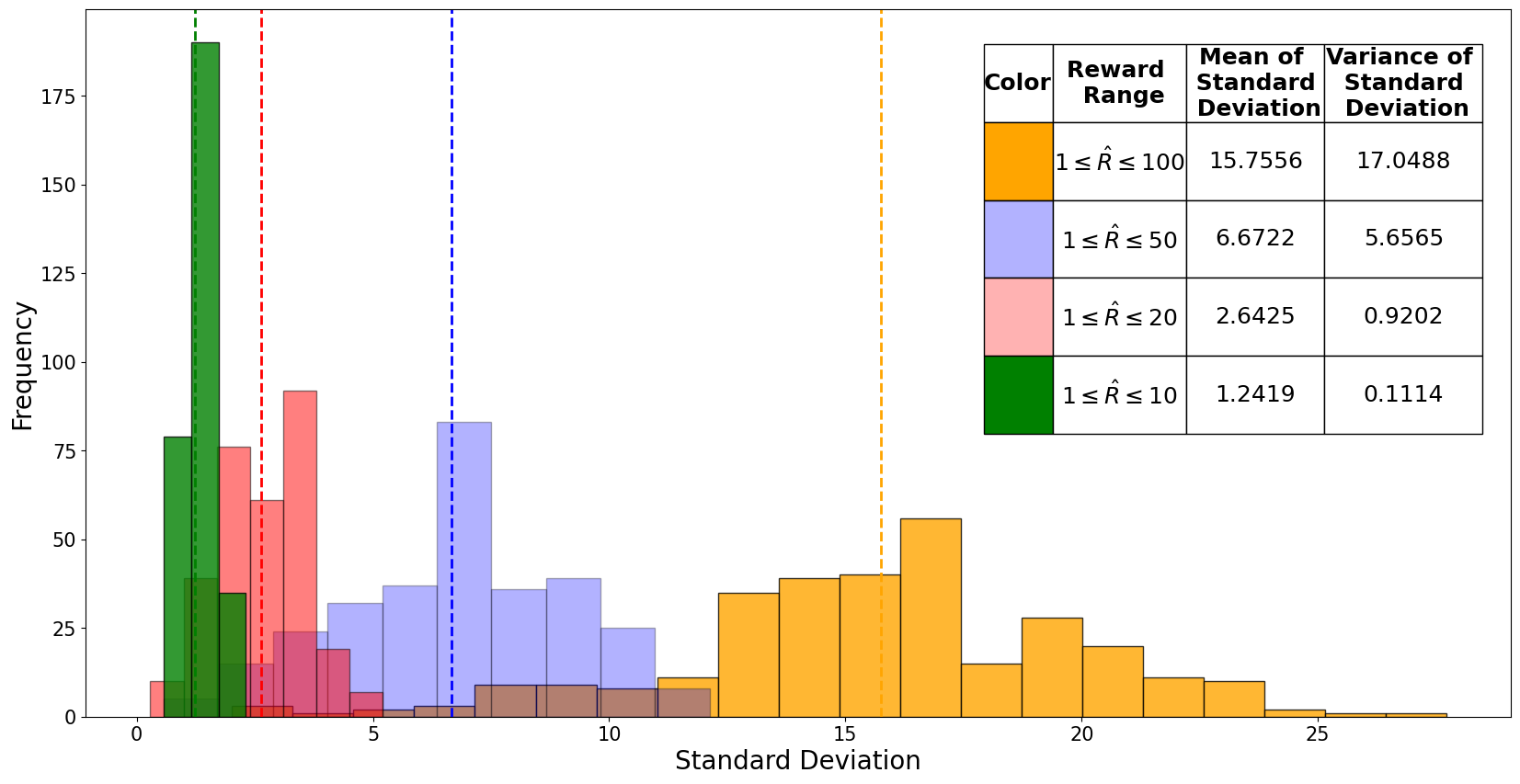}
    \captionof*{figure}{\footnotesize (e) The figure shows the distribution of the standard deviation of the rewards, $\sigma(x, y)$, computed across the entire dataset under different constraints on the allowable reward range. As the range of possible reward values decreases, the variability in the rewards assigned to different prompt–response pairs $(x, y)$ also diminishes. This leads to a reduction in the overall variability of the variances.}
\end{minipage}%
\caption{\footnotesize Empirical reward distributions for \textsc{Qwen2.5-0.5B} fine-tuned using the prompted \textsc{Gemini-2.0-flash} reward model under different settings of the reward range. We show the empirical reward distributions for the vanilla policies, variance-aware policies, and reference policy. The noisy reward $\hat{R}$ is sampled uniformly from $[a, b]$ with $1 \leq a < b \leq R$, where $R$ varies across $\{100, 50, 20, 10\}$. The true reward $r^\ast$ used in evaluation for each setting is defined as the midpoint $(a + b)/2$.}
\label{fig:variability-of-variance}
\end{figure}

\begin{table}[!htbp]
    \centering
    \footnotesize
    \begin{tabular}{|c|c|c|c|c|}
        \hline
        \vspace{0.1mm}
        \makecell{\textbf{Reward Range}} & \makecell{\textbf{Variance of std dev.} \\ \textbf{across prompts-responses}} & \makecell{\textbf{Vanilla Policy} \\ $\sigma_1^2$} & \makecell{\textbf{Variance-Aware Policy} \\ $\sigma_2^2$} & $\sigma_1^2/\sigma_2^2$ \\
        \hline
        $[1, 100]$ & 17 & 0.086 & 0.019 & 4.479 \\
        $[1, 50]$ & 5.6 & 0.2175 & 0.058 & 3.70 \\
        $[1, 20]$ & 0.92 & 0.0083 & 0.0033 & 2.55 \\
        $[1, 10]$ & 0.11 & 0.0017 & 0.0010 & 1.79 \\
        \hline
    \end{tabular}
    \caption{
\footnotesize
Comparison of the empirical reward variances under the vanilla and bariance-aware fine-tuning methods for \textsc{Qwen2.5-0.5B} with the prompted \textsc{Gemini-2.0-flash} reward model, across different reward ranges. The second column reports the variance of the standard deviation of rewards $\sigma(x,y)$ across prompt–response pairs. As the reward range narrows, the overall variability of the reward signal decreases, reducing the relative effectiveness of the variance-aware method, as seen in the decreasing ratio $\sigma_1^2/\sigma_2^2$.
}
    \label{tab:my_label}
\end{table}

%% file: Neurips/Proofs.tex
\section{Proofs}
\label{sec:Proofs}
\noisyreward*
\surrogate*
\begin{proof}
The result follows from a standard self-normalizing bound for Gaussian random variables. Specifically, for any $\delta > 0$, the following holds with probability at least $1 - \delta$:
\[
\left\| \hat{R} - r^* \right\|_{\Sigma^{-1}} \leq \sqrt{\abs{\cX}\abs{\cY} \cdot \ln\left(1/\delta\right)},
\]
since $\left\| \hat{R} - r^* \right\|_{\Sigma^{-1}}$ is the self-normalized Euclidean norm of a standard Gaussian vector in $\abs{\cX}\abs{\cY}$ dimensions.

Applying the Cauchy–Schwarz inequality for any $d \in \mathcal{D}$ yields:
\[
\left| \langle d, \hat{R} - r^* \rangle \right| \leq \|d\|_\Sigma \cdot \left\| \hat{R} - r^* \right\|_{\Sigma^{-1}}.
\]

Substituting the high-probability bound, we obtain:
\[
\left| \langle d, \hat{R} - r^* \rangle \right| \leq \|d\|_\Sigma \cdot \sqrt{\abs{\cX}\abs{\cY} \cdot \ln\left(1/\delta\right)}.
\]

This completes the proof.
\end{proof}

\risk*
\begin{proof}
Both optimization problems~\eqref{eq:pi1} and~\eqref{eq:pi2} involve maximizing a linear function over a convex domain. Hence, the maximum is achieved at the boundary of the feasible region, allowing us to replace the inequality constraints with equality constraints and apply the method of Lagrange multipliers.

\textbf{Variance-Aware Policy $\pi_2$.}  
The Lagrangian for the problem~\eqref{eq:pi2} is:
\[
\pi_2 = \arg\max_\pi \left[ \hat{R}^\top \pi - \beta (\pi - \pi_0)^\top \Sigma (\pi - \pi_0) \right],
\]
where $\beta$ is the Lagrange multiplier associated with the covariance-weighted $\ell_2$ constraint. Solving this yields:
\begin{equation}
\label{eq:pi2_solution}
\pi_2 = \pi_0 + \frac{1}{2\beta} \Sigma^{-1} \hat{R}.
\end{equation}

To enforce the constraint $\|\pi_2 - \pi_0\|_\Sigma^2 = \tilde{\epsilon}$, we compute $\beta$ as:
\[
\beta = \frac{1}{2} \sqrt{ \frac{ \hat{R}^\top \Sigma^{-1} \hat{R} }{ \tilde{\epsilon} } }.
\]
Substituting into~\eqref{eq:pi2_solution}, we get:
\[
\pi_2 = \pi_0 + \sqrt{ \frac{ \tilde{\epsilon} }{ \hat{R}^\top \Sigma^{-1} \hat{R} } } \Sigma^{-1} \hat{R}.
\]

\textbf{Variance-Unaware Policy $\pi_1$.}  
Similarly, solving~\eqref{eq:pi1} gives:
\[
\pi_1 = \pi_0 + \sqrt{ \frac{ \epsilon }{ \hat{R}^\top \hat{R} } } \hat{R}.
\]

\textbf{Expected True Rewards.}  
The expected true reward under $\pi_1$ is:
\[
\pi_1^\top r^* = \pi_0^\top r^* + \sqrt{ \frac{ \epsilon }{ \hat{R}^\top \hat{R} } } \hat{R}^\top r^*,
\]
and under $\pi_2$:
\[
\pi_2^\top r^* = \pi_0^\top r^* + \sqrt{ \frac{ \tilde{\epsilon} }{ \hat{R}^\top \Sigma^{-1} \hat{R} } } \hat{R}^\top \Sigma^{-1} r^*.
\]

Each policy underperforms relative to $\pi_0$ if its improvement term is non-positive. Thus, underperformance occurs when:
\begin{align*}
\pi_1: \quad &\hat{R}^\top r^* \leq 0, \\
\pi_2: \quad &\hat{R}^\top \Sigma^{-1} r^* \leq 0.
\end{align*}

Since $\hat{R} \sim \mathcal{N}(r^*, \Sigma)$, we compute:
\begin{align*}
\hat{R}^\top r^* &\sim \mathcal{N}\left( \| r^* \|^2, \; r^{*\top} \Sigma r^* \right), \\
\hat{R}^\top \Sigma^{-1} r^* &\sim \mathcal{N}\left( r^{*\top} \Sigma^{-1} r^*, \; r^{*\top} \Sigma^{-1} r^* \right).
\end{align*}

Hence, the underperformance probabilities are:
\begin{align*}
\mathbb{P}\left( \pi_1^\top r^* \leq \pi_0^\top r^* \right) &= \Phi\left( -\frac{ \| r^* \|^2 }{ \sqrt{ r^{*\top} \Sigma r^* } } \right), \\
\mathbb{P}\left( \pi_2^\top r^* \leq \pi_0^\top r^* \right) &= \Phi\left( -\sqrt{ r^{*\top} \Sigma^{-1} r^* } \right),
\end{align*}
where $\Phi(\cdot)$ denotes the standard normal cumulative distribution function.

\textbf{Comparing the Risks.}  
Using Cauchy–Schwarz:
\[
\|r^*\|^2 = r^{*\top} \Sigma^{-1/2} \Sigma^{1/2} r^* \leq \left\| \Sigma^{-1/2} r^* \right\| \cdot \left\| \Sigma^{1/2} r^* \right\| = \sqrt{ r^{*\top} \Sigma^{-1} r^* } \cdot \sqrt{ r^{*\top} \Sigma r^* }.
\]
Thus:
\[
-\frac{ \| r^* \|^2 }{ \sqrt{ r^{*\top} \Sigma r^* } } \geq -\sqrt{ r^{*\top} \Sigma^{-1} r^* },
\]
and by monotonicity of $\Phi(\cdot)$:
\[
\mathbb{P}\left( \pi_2^\top r^* \leq \pi_0^\top r^* \right) \leq \mathbb{P}\left( \pi_1^\top r^* \leq \pi_0^\top r^* \right).
\]
\end{proof}

\sharpe*
\begin{proof}
The constrained optimization problem can be converted into an unconstrained one via the method of Lagrange multipliers. Introducing $\beta > 0$ yields:
\[
\arg\max_{\pi} \mathbb{E}_{x \sim \rho_{\cX},\, y \sim \pi(\cdot|x)} \left[ \frac{ \hat{R}(x, y) }{ \beta \sigma^2(x, y) } - \ln \frac{ \pi(y|x) }{ \pi_0(y|x) } \right].
\]
This corresponds to a form of regularized exponential tilting, where the optimal policy becomes:
\[
\pi^*(y|x) \propto \pi_0(y|x) \exp\left( \frac{ \hat{R}(x, y) }{ \beta \sigma^2(x, y) } \right),
\]
which aligns with the notion of a Sharpe ratio: expected return scaled by risk. A full derivation is provided in~\cite[Appendix A.1]{rafailov2024direct}.
\end{proof}

%% file: Neurips/Reward_Model.tex
\section{Variability in Reward Models}
\label{sec:reward_model}

\subsection{Ensemble Reward Models: Finetuning and Observations}
In this section, we describe the process of reward modeling using the \textsc{Gemma-2b-it} model \citep{team2024gemma}, an instruction-tuned variant of the foundational \textsc{Gemma-2b}. Our methodology leverages an ensemble of 10 independently trained reward models to estimate the variance in rewards assigned to identical prompt–response pairs. This ensemble-based approach enables a more principled estimation of uncertainty in reward predictions and facilitates the analysis of model variability, even among models trained on the same data.

The sections that follow detail our methodology for training the ensemble of reward models, and key observations drawn from the ensemble's performance on multiple benchmark tasks.

\paragraph{Dataset}
To train our reward models, we use an open-source preference dataset introduced by \citet{dong2024rlhf}, available on HuggingFace\footnote{\href{https://huggingface.co/datasets/weqweasdas/preference_dataset_mix2}{huggingface.co/weqweasdas/preference\_dataset\_mix2}}. This curated dataset comprises approximately 50{,}000 labeled preference pairs and is constructed by merging several widely-used open datasets. The constituent datasets include \textsc{HH-RLHF} \citep{bai2022training}, \textsc{SHP} \citep{ethayarajh2022understanding}, \textsc{HelpSteer} \citep{wang2023helpsteer}, \textsc{PKU-SafeRLHF} \citep{ji2024beavertails}, \textsc{UltraFeedback} \citep{cui2023ultrafeedback}, \textsc{UltraInteract} \citep{yuan2024advancing}, \textsc{Distilabel-Capybara} \citep{daniele2023suphavadeeprasit}, and \textsc{Distilabel-Orca3} \citep{lian2023openorca}.

To ensure data quality, the combined dataset is preprocessed to remove approximately 10\% of low-quality samples. The resulting dataset contains prompt–response preference pairs, where each prompt is associated with one preferred and one rejected response. These pairwise labels serve as the ground truth for training the ensemble reward models.

The diversity and breadth of the included datasets make this collection well-suited for training robust reward models that generalize across a range of domains and tasks. For further details on dataset construction and preprocessing, we refer the reader to \citet{dong2024rlhf}.

\paragraph{Methodology}  
We use the \textsc{Gemma-2b-it} model \citep{gemma-2b-it} as the backbone for our reward models. As an instruction-tuned variant of the foundational \textsc{Gemma-2b}, it is well-suited for reward modeling tasks due to its alignment with human instructions. The full model, including a scalar reward head, occupies approximately 9.34~GB on disk. Since we employ an ensemble of 10 independently parameterized reward models, the total storage requirement scales to roughly 90~GB.

To accelerate training and improve memory efficiency, we adopt the following methodology:
\begin{itemize}
    \item \textbf{Initial Training:} We begin by training a single instance of the full \textsc{Gemma-2b-it} model with a scalar reward head—a linear layer of size $2048 \times 1$—on the preference dataset. We apply early stopping once the loss reaches $0.3$ to balance model capacity and mitigate overfitting.

    \item \textbf{Parallel Reward Heads:} After partial training, we attach nine additional reward heads in parallel to the original, following the setup of \citet{zhang2024improving}. Each head has identical dimensions ($2048 \times 1$), resulting in a 10-dimensional output vector in a single forward pass—one scalar reward per ensemble member.

    \item \textbf{Freezing the Foundation Model:} To reduce computational costs, we freeze the pre-trained base of \textsc{Gemma-2b-it} and train only the reward heads. This setup simulates training 10 independent reward models in parallel while sharing the same foundation. We employ an additive loss function:
    \[
        \text{loss} = \sum_{i=1}^{10} \ell(\theta_i),
    \]
    where each $\theta_i$ denotes the parameters of the $i$-th reward head. This ensures each head is trained independently, while retaining efficiency.
\end{itemize}

By freezing the foundational layers and isolating training to the reward heads, we significantly reduce both computational and storage overhead. The final ensemble model remains compact, occupying approximately 9.34~GB on disk, with only 20,480 trainable parameters across all reward heads.

\begin{figure*}[t]
    \makebox[1.02\textwidth][c]{%
        \begin{minipage}{1.1\textwidth} % Adjust 1.2 to control the extension
            \centering
            \begin{subfigure}[\footnotesize{\textbf{Chat}}]{\includegraphics[width=0.25\linewidth]{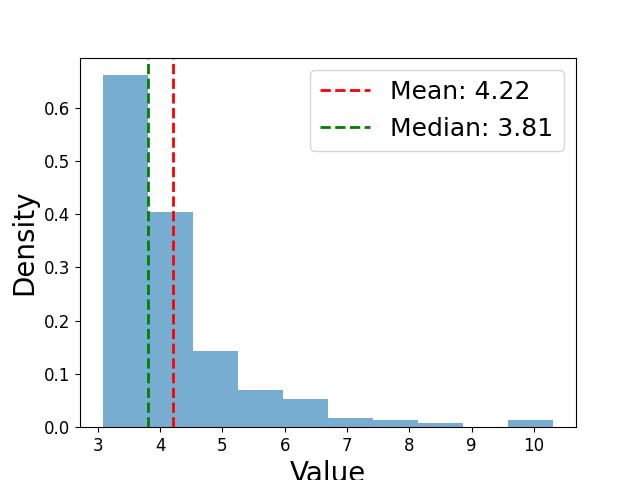}}   \end{subfigure}
            \hspace*{-0.3cm}
            \begin{subfigure}[\footnotesize{\textbf{Chat Hard}}]{\includegraphics[width=0.25\linewidth]{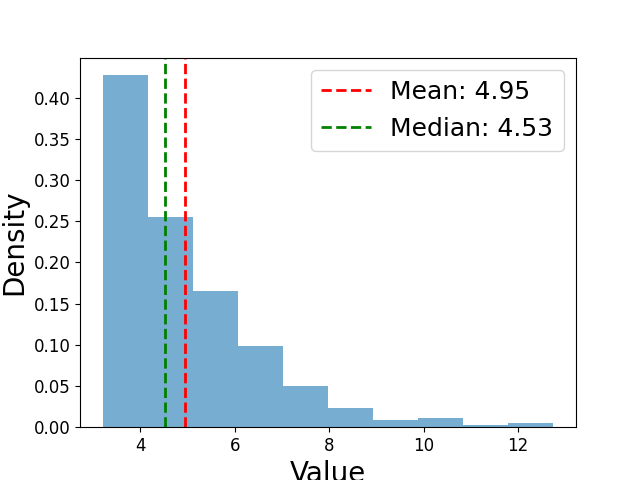}}   
            \end{subfigure}
            \hspace*{-0.3cm}
            \begin{subfigure}[\footnotesize{\textbf{Safety}}]{\includegraphics[width=0.25\linewidth]{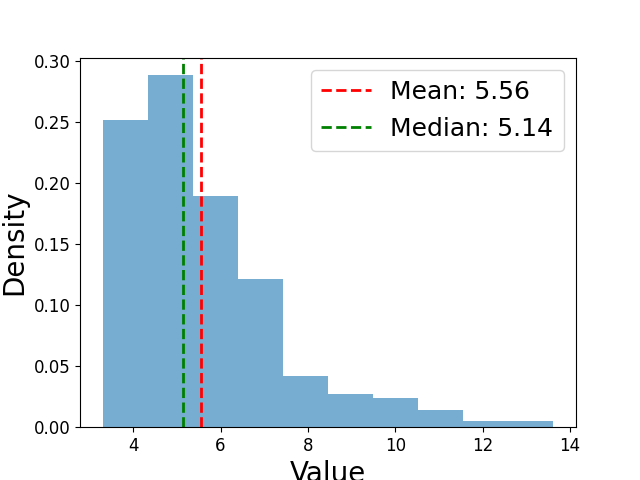}}  \end{subfigure}
            \hspace*{-0.3cm}
            \begin{subfigure}[\footnotesize{\textbf{Reasoning}}]{\includegraphics[width=0.25\linewidth]{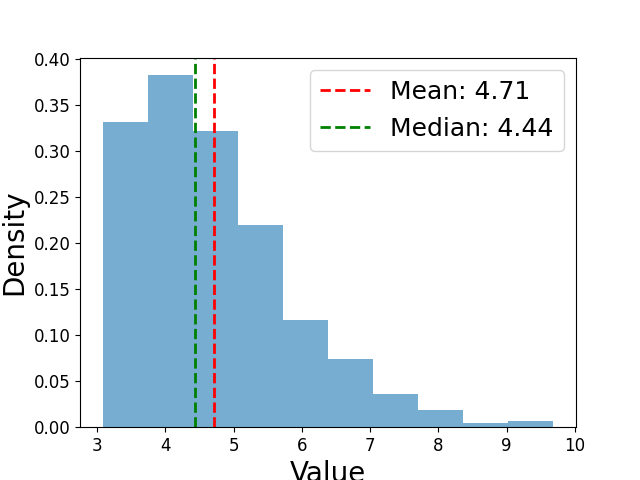}}   
            \end{subfigure} \\
            \begin{subfigure}[\footnotesize{\textbf{Chat}}]{\includegraphics[width=0.25\linewidth]{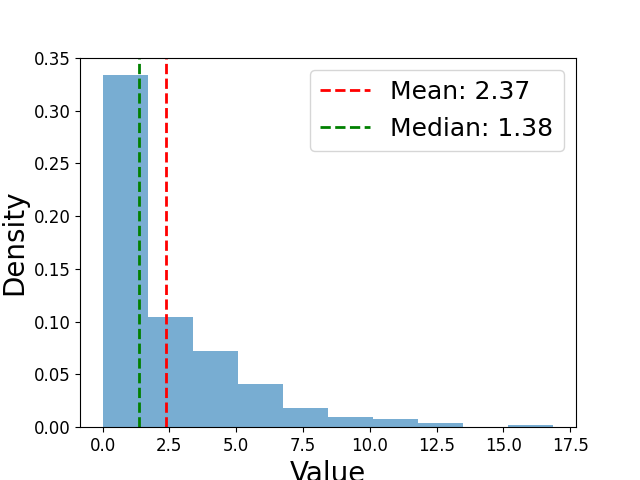}}   \end{subfigure}
            \hspace*{-0.3cm}
            \begin{subfigure}[\footnotesize{\textbf{Chat Hard}}]{\includegraphics[width=0.25\linewidth]{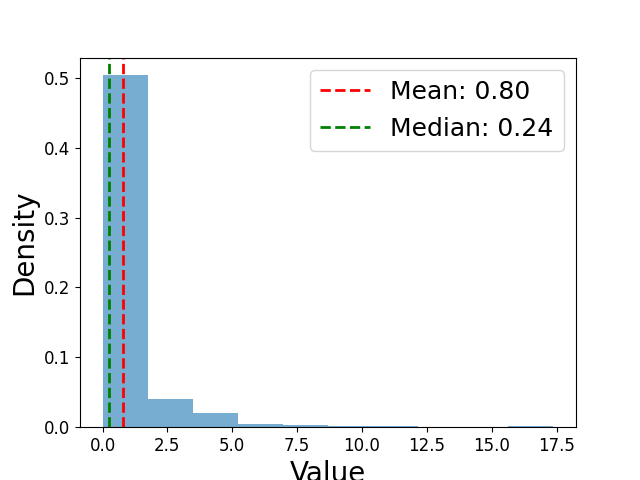}}   
            \end{subfigure}
            \hspace*{-0.3cm}
            \begin{subfigure}[\footnotesize{\textbf{Safety}}]{\includegraphics[width=0.25\linewidth]{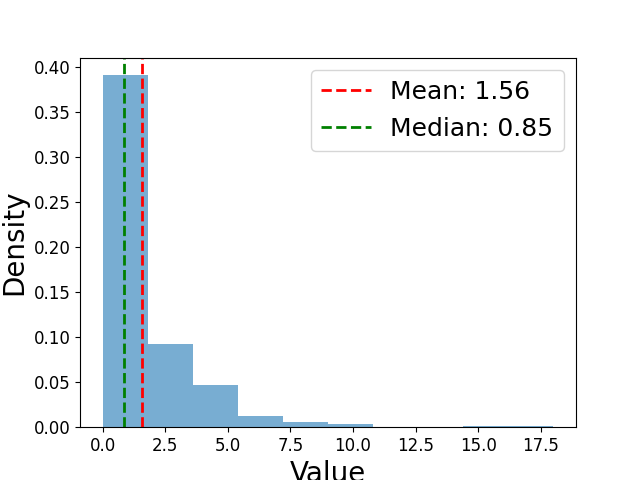}}  \end{subfigure}
            \hspace*{-0.3cm}
            \begin{subfigure}[\footnotesize{\textbf{Reasoning}}]{\includegraphics[width=0.25\linewidth]{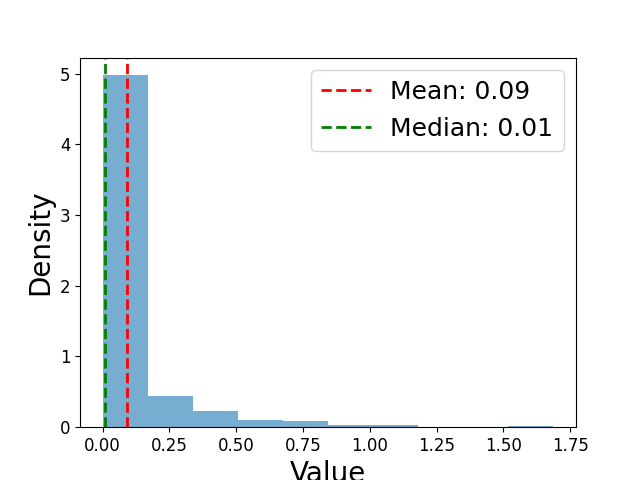}}   
            \end{subfigure}
        \end{minipage}%
    }
    \caption{\footnotesize (\textit{Top Row}) Distribution of sample variances in the rewards assigned to accepted responses. Each variance is computed across the outputs of 10 independently trained reward models. The median sample variance exceeds $3.81$ in all datasets, with values reaching as high as $10$, confirming substantial model-level variability in reward estimates for the same prompt–response pair. 
    (\textit{Bottom Row}) Distribution of the sample variances of reward differences between accepted and rejected responses. The spread in these distributions indicates that the models are not simple scalar translations of one another. Instead, the variance reflects genuine disagreement driven by the statistical and stochastic nature of training.}
    %\vspace*{-0.1cm}
    \label{fig:reward_modelling}
\end{figure*}

\paragraph{Observations} 
To investigate whether identically trained reward models exhibit disagreement on the same prompt–response pairs, we evaluate our ensemble on four datasets from the \textsc{RewardBenchmark} platform: \textsc{Chat}, \textsc{Chat-Hard}, \textsc{Safety}, and \textsc{Reasoning}. As an example, the \textsc{Chat} dataset comprises 358 prompt–response pairs of the form $(x, y^1, y^2)$, where $y^1$ is the preferred response and $y^2$ is the rejected one. This dataset is constructed from a mixture of sources, including \textsc{AlpacaEval Easy}, \textsc{AlpacaEval}, \textsc{AlpacaEval Hard}~\citep{li2023alpacaeval}, \textsc{MT Bench Easy}, and \textsc{MT Bench Medium}~\citep{zheng2023judging}. Details of the other benchmark compositions are provided in \citet{lambert2024rewardbench}.

For each accepted response $y^1$, we compute the reward $r_i(x, y^1)$ using each of the $10$ reward models and calculate the sample variance across these outputs. The top row of Figure~\ref{fig:reward_modelling} shows the distribution of these sample variances for all datasets. We observe that the reward variances range from $3$ to $14$, with a mean above $4$ and a median above $3$ in each case—indicating substantial variability across models, despite being trained on the same data. This disagreement arises due to factors such as limited data and the stochasticity inherent in optimization procedures.

To further assess whether these models differ beyond scalar shifts, we examine the variance in the reward difference between accepted and rejected responses, i.e., $r_i(x, y^1) - r_i(x, y^2)$. The bottom row of Figure~\ref{fig:reward_modelling} shows that this variance is also substantial, ruling out the possibility that models differ by constant offsets alone. If the models were simple translations of one another, we would expect the distribution to collapse to a Dirac delta at zero, which is clearly not the case. These results confirm that ensemble disagreement reflects meaningful variance in learned preferences, rather than superficial noise.

\subsection{Hyperparameter Details for Reward Model Finetuning}

The hyperparameter configuration used for training the single reward-head model is listed in Table~\ref{tab:hyperparams}, with all remaining settings aligned with those reported in \citet{wolf-etal-2020-transformers}. Table~\ref{tab:hw_requirements} provides a summary of the hardware specifications and resource utilization during training, including GPU memory, disk usage, and wall-clock time. The training was conducted on four NVIDIA A40 GPUs (48 GB each), with approximately 30 GB of disk space required for the dataset, model checkpoints, and logs. The total training time was 51 hours.
\begin{table*}[!htbp]
    \centering
    \begin{minipage}{0.45\textwidth}
        \centering
        \begin{adjustbox}{max width=\textwidth}
        \begin{tabular}{|c|c|}
            \hline
            \textbf{Hyperparameter}    & \textbf{Value} \\ \hline
            Effective Batch Size       & 32             \\ \hline
            Learning Rate              & 1e-5           \\ \hline
            Optimizer                  & Paged AdamW 32bit \\ \hline
            Weight Decay               & 0.001          \\ \hline
            LR Scheduler               & cosine         \\ \hline
            Epochs                     & 1              \\ \hline
            Global Train Steps         & 4125           \\ \hline
        \end{tabular}    
        \end{adjustbox}
        \caption{\footnotesize{Hyperparameters used in training the Single Reward Model.}}
        \label{tab:hyperparams}
    \end{minipage}%
    \hspace{0.04\textwidth} % Space between the tables
    \begin{minipage}{0.5\textwidth}
        \centering
        \begin{adjustbox}{max width=\textwidth}
        \begin{tabular}{|c|c|}
            \hline
            \textbf{Resource}          & \textbf{Details} \\ \hline
            GPU Model                  & NVIDIA A40 (48 GB) \\ \hline
            Number of GPUs             & 4               \\ \hline
            Total GPU Memory           & 12.68 GB        \\ \hline
            Total Disk Space Required  & 30 GB           \\ \hline
            Total Training Time        & 51 hours        \\ \hline
        \end{tabular}
        \end{adjustbox}
        \caption{\footnotesize{Hardware requirements for training the single reward model.}}
        \label{tab:hw_requirements}
    \end{minipage}
\end{table*}

The hyperparameter settings used for ensemble reward modeling are provided in Table~\ref{tab:hyperparams_ensemble}, with all other configurations following \citet{wolf-etal-2020-transformers}. Table~\ref{tab:ensemble-hardware} summarizes the hardware specifications and resource usage during ensemble training, including GPU memory, disk requirements, and total training duration. The ensemble was trained on four NVIDIA A40 GPUs (each with 48GB of memory). The full training process required approximately 40GB of disk space for datasets, checkpoints, and logs, and completed in 7~hours.

\begin{table*}[!htbp]
    \centering
    \begin{minipage}{0.45\textwidth}
        \centering
        \resizebox{\textwidth}{!}{%
        \begin{tabular}{|c|c|}
            \hline
            \textbf{Hyperparameter}    & \textbf{Value} \\ \hline
            Effective Batch Size       & 32             \\ \hline
            Learning Rate              & 1e-5           \\ \hline
            Optimizer                  & Paged AdamW 32bit \\ \hline
            Weight Decay               & 0.001          \\ \hline
            LR Scheduler               & cosine         \\ \hline
            Epochs                     & 0.5            \\ \hline
            Global Train Steps         & 2060           \\ \hline
        \end{tabular}%
        }
        \caption{\footnotesize{Hyperparameters used in training the Ensemble Reward Model.}}
        \label{tab:hyperparams_ensemble}
    \end{minipage}%
    \hspace{0.04\textwidth} % Space between the tables
    \begin{minipage}{0.50\textwidth}
        \centering
        \resizebox{0.9\textwidth}{!}{%
        \begin{tabular}{|c|c|}
            \hline
            \textbf{Resource}          & \textbf{Details} \\ \hline
            GPU Model                  & NVIDIA A40       \\ \hline
            Number of GPUs             & 4               \\ \hline
            Total GPU Memory           & 6.12 GB         \\ \hline
            Total Disk Space Required  & 38 GB           \\ \hline
            Total Training Time        & 7 hours         \\ \hline
        \end{tabular}%
        }
        \caption{\footnotesize{Hardware requirements for training the ensemble reward model.}}
        \label{tab:ensemble-hardware}
    \end{minipage}
\end{table*}

%% file: Neurips/Reward_Prompting.tex
\section{Prompting Reward Models}
\label{app:prompt}
This section describes the prompt used to extract confidence intervals from \textit{Gemini-1.5-flash}, \textit{Gemini-2.0-flash}, and \textit{Deepseek-V3}.
\begin{center}
\begin{tcolorbox}[
  enhanced,
  colback=lightgray,
  colframe=lightgray,
  boxrule=0pt,
  sharp corners,
  width=0.9\textwidth,
  ]
  You represent a \textit{diverse} population of human beings across:\\
 
- Different Age Groups (10 to 60)\\
- Different \textit{moralities}, \textit{ethics} and \textit{religions}\\
- Different Genders and Sexual orientations\\
- Different \textit{Cultural}, \textit{Linguistic}, \textit{National}, \textit{Ethnic} backgrounds\\
- Different \textit{Educational}, \textit{Professional} and \textit{Economic} backgrounds\\
- Different \textit{Physical} and \textit{Mental} abilities\\
- Different \textit{Political} and \textit{Philosophical} beliefs\\
- Different \textit{Social} and \textit{Familial} structures\\
- Different and Alternative \textit{Historical} and \textit{Futuristic} contexts\\
- Different \textit{Technological} and \textit{Scientific} contexts\\
- Different \textit{Artistic} and \textit{Creative} contexts\\
- Neurodiversity, cognitive styles, and artistic sensibilities\\
 \\
Your task is to evaluate the aesthetic value of the following text. \\``Aesthetic value" refers to how stylistically pleasant, expressive, emotionally resonant, or creatively engaging the text feels to different human readers — regardless of factual correctness or usefulness. Note that aesthetic value is subjective and can vary widely across individuals and cultures — there are no right or wrong answers.\\
 \\
Return two real numbers `a' and `b' such that:\\
- $1 \leq a < b \leq 100$\\
- They represent a \textit{reasonably loose lower and upper bound} on the range of likely aesthetic judgments across the population described \\
- If the population is likely to disagree strongly, widen the range\\
- If the population is likely to agree strongly, narrow the range\\
- Do not include any commentary or explanation — just return the numbers on separate lines\\
 
-----\\
Text: "{text}"\\
-----\\
 
a:\\
b:\\
\end{tcolorbox}
\end{center}

%% file: Neurips/PPO.tex
\section{Proximal Policy Optimization (PPO)}
\label{sec:PPO}
\subsection{Methodology}
\paragraph{Dataset}  
We conduct PPO (as well as variance-aware PPO) training using a dataset constructed by combining the following open-source datasets: \textsc{alpacaeval-easy, alpacaeval-length, alpacaeval-hard, mt-bench-easy, mt-bench-medium}. From each entry in these datasets, we extract only the \texttt{prompt} field, which serves as the input to the language model during fine-tuning.

\begin{algorithm}[!htbp]
\footnotesize
\caption{Fine-Tuning and Collecting Policies by Method Type}
\label{alg:multi_policy_finetune}
\begin{algorithmic}[1]
\State \textbf{Global:} \textsc{Uncertainty\_Type} $\in \{\textsc{Prompted}, \textsc{Ensemble}\}$
\Statex \hspace{3em} \textsc{Methods} $= \{\textsc{VarianceAware}, \textsc{Vanilla} \}$

\State \textbf{Input:} 
\Statex \hspace{3em} \textsc{Dataset} $\mathcal{D}$, \textsc{Model} $\mathcal{M}$, \textsc{Reward Model} $\mathcal{R}$
\State \textbf{Output:} 
\Statex \hspace{3em} Policy sets: $\Pi_1$ (Vanilla), $\Pi_2$ (Variance-Aware)

\Procedure{MultiPolicyFinetune}{$\mathcal{D}, \mathcal{M}, \mathcal{R}$}
    \State Initialize policy lists: $\Pi_1 \gets [\,], \Pi_2 \gets [\,]$
    \State Define base policy: $\pi_0(x) \gets \mathcal{M}(x)$
    
    \For{\textsc{Method} in \textsc{Methods}}
        \For{$i = 1$ to $N$}
            \If{\textsc{Method} is \textsc{VarianceAware}}
                \State $\pi \gets \textsc{PPOUpdate}(\mathcal{D}, \mathcal{M}, \mathcal{R})$ \Comment{Algorithm~\ref{alg:ppo}}
                \State $\Pi_2 \gets \Pi_2 \cup \{\pi\}$
            \ElsIf{\textsc{Method} is \textsc{Vanilla}}
                \State $\pi \gets \textsc{PPOUpdate}(\mathcal{D}, \mathcal{M}, \mathcal{R})$ \Comment{Algorithm~\ref{alg:ppo}}
                \State $\Pi_1 \gets \Pi_1 \cup \{\pi\}$
            \EndIf
        \EndFor
    \EndFor

    \State \Return $(\Pi_1, \Pi_2)$
\EndProcedure
\end{algorithmic}
\end{algorithm}

\begin{algorithm}[!htbp]
\footnotesize
\caption{Evaluation Phase for Fine-Tuned and Reference Policies}
\label{alg:evaluate_finetuned_policies}
\begin{algorithmic}[1]
\State \textbf{Global:} \textsc{Uncertainty\_Type} $\in \{\textsc{Prompted}, \textsc{Ensemble}\}$

\State \textbf{Input:}
\Statex \hspace{3em} \textsc{Dataset} $\mathcal{D}$, \textsc{Model} $\mathcal{M}$, \textsc{Reward Model} $\mathcal{R}$, 
\Statex \hspace{3em} Policy sets: $\Pi_1$ (Vanilla), $\Pi_2$ (Variance-Aware) \Comment{Algorithm~\ref{alg:multi_policy_finetune}}

\State \textbf{Output:} 
\Statex \hspace{3em} Average true reward for each policy set: $\Bar{r}_{\Pi_1}$, $\Bar{r}_{\Pi_2}$, $\Bar{r}_{\Pi_0}$

\Procedure{EvaluatePolicies}{$\mathcal{D}, \mathcal{M}, \mathcal{R}, \Pi_1, \Pi_2$}
    \State Initialize average reward lists: $\Bar{r}_{\Pi_1} \gets [\,], \Bar{r}_{\Pi_2} \gets [\,], \Bar{r}_{\Pi_0} \gets [\,]$
    
    \For{each policy set $\Pi \in \{\Pi_1, \Pi_2\}$}
        \For{each policy $\pi \in \Pi$}
            \State Initialize reward list $\mathcal{R}_\pi \gets [\,]$
            \For{each prompt $x \in \mathcal{D}$}
                \State Generate response $y \gets \pi(x)$
                \State Concatenate $s \gets x \mathbin\Vert y$
                \State $\mathcal{R}(s) \gets \textsc{RewardModel}(s)$ \Comment{Algorithm~\ref{alg:reward_model}}
                \State $r^*(s) \gets \textsc{GetTrueReward}(\mathcal{R}(s))$ \Comment{Algorithm~\ref{alg:true_reward}}
                \State Append $r^*(s)$ to $\mathcal{R}_\pi$
            \EndFor
            \State Compute and append average: 
            $\Bar{r}_\Pi \gets \Bar{r}_\Pi \cup \left( \frac{1}{|\mathcal{R}_\pi|} \sum_{r \in \mathcal{R}_\pi} r \right)$
        \EndFor
    \EndFor

    % \Comment{Evaluate the reference policy \(\pi_0(x) = \mathcal{M}(x)\) for \(|\Pi_1|\) trials}
    % \State \textit{Evaluate the reference policy $\pi_0(x) = \mathcal{M}(x)$ for $|\Pi_1|$ trials}

    \For{$i = 1$ to $|\Pi_1|$}
    \Comment{Evaluate the reference policy \(\pi_0(x) = \mathcal{M}(x)\) for \(|\Pi_1|\) trials}
        \State Initialize reward list $\mathcal{R}_{\pi_0}^i \gets [\,]$
        \For{each prompt $x \in \mathcal{D}$}
            \State Generate response $y \gets \mathcal{M}(x)$
            \State Concatenate $s \gets x \mathbin\Vert y$
            \State $\mathcal{R}(s) \gets \textsc{RewardModel}(s)$ \Comment{Algorithm~\ref{alg:reward_model}}
            \State $r^*(s) \gets \textsc{GetTrueReward}(\mathcal{R}(s))$ \Comment{Algorithm~\ref{alg:true_reward}}
            \State Append $r^*(s)$ to $\mathcal{R}_{\pi_0}^i$
        \EndFor
        \State Compute average: $\Bar{r}_{\pi_0}^i \gets \frac{1}{|\mathcal{R}_{\pi_0}^i|} \sum_{r \in \mathcal{R}_{\pi_0}^i} r$
        \State Append $\Bar{r}_{\pi_0}^i$ to $\Bar{r}_{\Pi_0}$
    \EndFor
    
    \State \Return $\Bar{r}_{\Pi_1}, \Bar{r}_{\Pi_2}, \Bar{r}_{\Pi_0}$
\EndProcedure
\end{algorithmic}
\end{algorithm}

\begin{algorithm}[!htbp]
\footnotesize
\caption{Unified PPO Update for Vanilla and Variance-Aware Fine-Tuning\label{alg:ppo}}
\begin{algorithmic}[1]
\State \textbf{Global:} \textsc{Uncertainty\_Type} $\in \{\textsc{Prompted}, \textsc{Ensemble}\}$
\Statex \hspace{3em} \textsc{Methods} $= \{\textsc{VarianceAware}, \textsc{Vanilla}\}$
\State \textbf{Input:} 
\Statex \hspace{3em} \textsc{Dataset} $\mathcal{D}$, \textsc{Model} $\mathcal{M}$, \textsc{Reward Model} $\mathcal{R}$
\State \textbf{Output:} 
\State \hspace{3em} Fine-tuned policy $\pi_\theta$
\Procedure{PPOUpdate}{$\mathcal{M}, \mathcal{R}$}
    \State Define $\pi_0(x) \gets \mathcal{M}(x)$
    \State Initialize $\pi_\theta(x) \gets \pi_0(x)$
    \For{$\text{iteration} = 1$ to $M$}
        \State Sample batch $\mathcal{D}_B = \{x_i\}_{i=1}^B$
        \For{each prompt $x \sim \mathcal{D}_B$}
            \State Generate response $y \gets \mathcal{M}(x)$
            \State Concatenate $s \gets x \mathbin\Vert y$
            \State $\mathcal{R}(s) \gets \textsc{RewardModel}(s)$ \Comment{Algorithm~\ref{alg:reward_model}}
            \State Get reward: $\hat{r}(s) \gets \textsc{Sample}(\mathcal{R}(s))$ \Comment{Algorithm~\ref{alg:sample_reward}}

            \If{\textsc{Method} is \textsc{VarianceAware}}
                \State Get variance: $\sigma^2(x, y) \gets \textsc{GetVariance}(\mathcal{R}(s))$ \Comment{Algorithm~\ref{alg:compute_variance}}
                \State Compute loss $\cL(\theta)$ using Eq.~\eqref{eq:variance_aware_loss}
            \Else
                \State Compute loss $\cL(\theta)$ using Eq.~\eqref{eq:vanilla_loss}
                \EndIf
        \EndFor
        \State Update $\theta$ via gradient ascent on $\mathcal{L}(\theta)$
    \EndFor
    \State \Return $\pi_\theta$
\EndProcedure
\end{algorithmic}
\end{algorithm}

\begin{algorithm}[!htbp]
\footnotesize
\caption{Reward Model Interface\label{alg:reward_model}}
\begin{algorithmic}[1]
\State \textbf{Global:} \textsc{Uncertainty\_Type} $\in \{\textsc{Prompted}, \textsc{Ensemble}\}$
\State \textbf{Input:} 
\State \hspace{1.2em} Concatenated string $s = x \mathbin\Vert y$ 
\State \textbf{Output:} 
\State \hspace{1.2em} Reward response $\mathcal{R}(s)$
\Procedure{RewardModel}{$s$}
    \If{\textsc{Uncertainty\_Type} is \textsc{Prompted}}
        \State \Return $[a, b] \gets \textsc{CallModel}(\textcolor{blue}{\textsc{Reward\_Prompt}} \mathbin\Vert s)$ \Comment{\scriptsize{Concatenate \textcolor{blue}{\textsc{Reward\_Prompt}} and $s$.}}
    \footnotesize    
    \ElsIf{\textsc{Uncertainty\_Type} is \textsc{Ensemble}}
        \State \Return $\{r_1, \dots, r_{10}\} \gets \textsc{CallModel}(s)$ \Comment{List of scalar rewards}
    \EndIf
\EndProcedure
\end{algorithmic}
\end{algorithm}

\begin{algorithm}[!htbp]
\footnotesize
\caption{Sampling from Reward Model Output\label{alg:sample_reward}}
\begin{algorithmic}[1]
\State \textbf{Global:} \textsc{Uncertainty\_Type} $\in \{\textsc{Prompted}, \textsc{Ensemble}\}$
\State \textbf{Input:} 
\State \hspace{1.2em} Reward output $\mathcal{R}(s)$ \Comment{Algorithm~\ref{alg:reward_model}}
\State \textbf{Output:} 
\State \hspace{1.2em} Sampled reward $\hat{r}(s)$
\Procedure{Sample}{$\mathcal{R}(s)$}
    \If{\textsc{Uncertainty\_Type} is \textsc{Prompted}}
        \State Let $[a, b] \gets \mathcal{R}(s)$
        \State \Return $\textsc{Uniform}(a, b)$
    \ElsIf{\textsc{Uncertainty\_Type} is \textsc{Ensemble}}
        \State Let $\{r_1, \dots, r_{10}\} \gets \mathcal{R}(s)$
        \State \Return $\frac{1}{10} \sum_{i=1}^{10} r_i$
    \EndIf
\EndProcedure
\end{algorithmic}
\end{algorithm}

\begin{algorithm}[!htbp]
\footnotesize
\caption{Computing Variance from Reward Model Output\label{alg:compute_variance}}
\begin{algorithmic}[1]
\State \textbf{Global:} \textsc{Uncertainty\_Type} $\in \{\textsc{Prompted}, \textsc{Ensemble}\}$
\State \textbf{Input:} 
\State \hspace{1.2em} Reward output $\mathcal{R}(s)$
\Comment{Algorithm~\ref{alg:reward_model}}
\State \textbf{Output:} 
\State \hspace{1.2em} Variance estimate $\sigma^2(s)$
\Procedure{GetVariance}{$\mathcal{R}(s)$}
    \If{\textsc{Uncertainty\_Type} is \textsc{Prompted}}
        \State Let $[a, b] \gets \mathcal{R}(s)$
        \State \Return $\frac{(b - a)^2}{12}$ 
    \ElsIf{\textsc{Uncertainty\_Type} is \textsc{Ensemble}}
        \State Let $\{r_1, \dots, r_{10}\} \gets \mathcal{R}(s)$
        \State Compute sample mean: $\bar{r} \gets \frac{1}{10} \sum_{i=1}^{10} r_i$
        \State \Return $\frac{1}{9} \sum_{i=1}^{10} (r_i - \bar{r})^2$
    \EndIf
\EndProcedure
\end{algorithmic}
\end{algorithm}

\begin{algorithm}[!htbp]
\footnotesize
\caption{Computing Ground-Truth Reward\label{alg:true_reward}}
\begin{algorithmic}[1]
\State \textbf{Global:} \textsc{Uncertainty\_Type} $\in \{\textsc{Prompted}, \textsc{Ensemble}\}$
\State \textbf{Input:} 
\State \hspace{1.2em} Reward output $\mathcal{R}(s)$
\State \textbf{Output:} 
\State \hspace{1.2em} Ground-truth reward $r^*(s)$
\Procedure{GetTrueReward}{$\mathcal{R}(s)$}
    \If{\textsc{Uncertainty\_Type} is \textsc{Prompted}}
        \State Let $[a, b] \gets \mathcal{R}(s)$
        \State \Return $r^*(s) \gets (a + b)/2$
    \ElsIf{\textsc{Uncertainty\_Type} is \textsc{Ensemble}}
        \State \Return $r^*(s) \gets \textsc{FsfairX-LLaMA3-RM-v0.1}(s)$ \Comment{\footnotesize Call external large reward model}
    \EndIf
\EndProcedure
\end{algorithmic}
\end{algorithm}

\paragraph{Methodology}  
We fine-tune \textsc{GPT-2}~\cite{radford2019language}, \textsc{Qwen2.5-0.5B}~\cite{Yang2024Qwen25TR}, and \textsc{Mistral-7B}~\cite{jiang2023mistral} using uncertainty-aware reward modeling. Specifically, \textsc{GPT-2} is paired with a custom scalar reward model as described in Appendix~\ref{sec:reward_model}, while \textsc{Qwen2.5-0.5B} and \textsc{Mistral-7B} are aligned using \textsc{Prompted} reward models constructed from \textsc{Gemini-1.5-flash}, \textsc{Gemini-2.0-flash}, and \textsc{Deepseek-V3}.

We introduce a variance-aware fine-tuning framework that explicitly incorporates uncertainty estimates from the reward model into the policy optimization process. Our approach is compared against the standard PPO baseline across two classes of reward uncertainty:

\begin{itemize}[leftmargin=1.5em]
    \item \textbf{\textsc{Prompted}}: The reward model is queried via a structured natural language prompt (Appendix~\ref{app:prompt}), returning an interval-valued reward $[a, b]$ for a given prompt-response pair. The reward is assumed to be uniformly distributed over this interval.
    
    \item \textbf{\textsc{Ensemble}}: The reward model consists of an ensemble of independently trained scalar-valued reward heads as described in Appendix~\ref{sec:reward_model}. For each input, the ensemble returns a list of scalar rewards $\{r_1, \dots, r_{10}\}$. The sample mean and sample variance are used as the reward estimate and uncertainty measure, respectively.
\end{itemize}

For each reward model configuration, we fine-tune a total of \( N = 80 \) policies using two methods:
\begin{itemize}[leftmargin=1.5em, nosep]
    \item \textsc{Vanilla PPO}: Optimizes the policy using sampled scalar rewards without accounting for reward variance.
    \item \textsc{Variance-Aware PPO}: Incorporates reward variance into the PPO objective to regularize updates more cautiously under uncertainty.
\end{itemize}

The corresponding optimization objectives are:
\begin{equation}
\label{eq:variance_aware_loss}
\mathcal{L}_{\textsc{VA}}(\theta) = \sum_{x, y} \left[ \hat{r}(x, y) - \sigma^2(x, y) \log \frac{\pi_\theta(y|x)}{\pi_0(y|x)} \right]
\end{equation}
\begin{equation}
\label{eq:vanilla_loss}
\mathcal{L}_{\textsc{Vanilla}}(\theta) = \sum_{x, y} \left[ \hat{r}(x, y) - \log \frac{\pi_\theta(y|x)}{\pi_0(y|x)} \right]
\end{equation}

The high-level fine-tuning workflow is described in Algorithm~\ref{alg:multi_policy_finetune} and proceeds as follows:

\begin{enumerate}[leftmargin=1.5em]
    \item For each training iteration and for each policy:
    \begin{enumerate}[nosep]
        \item Sample a prompt \( x \in \mathcal{D} \), and generate a response \( y \sim \pi_\theta(\cdot|x) \).
        \item Form a string \( s = x \mathbin\Vert y \) and send it to the reward model.
        \item The reward model returns either:
        \begin{itemize}[nosep]
            \item an interval $[a, b]$ if \textsc{Prompted}, with reward sampled as $\hat{r}(x, y) \sim \mathcal{U}[a, b]$, and variance $\sigma^2(x, y) = \frac{(b - a)^2}{12}$;
            \item a list of values $\{r_1, \dots, r_{10}\}$ if \textsc{Ensemble}, with $\hat{r}(x, y)$ as the mean and $\sigma^2(x, y)$ as the empirical variance.
        \end{itemize}
        \item Update the policy via PPO using Equation~\eqref{eq:vanilla_loss} or~\eqref{eq:variance_aware_loss}, depending on the chosen method.
    \end{enumerate}
    \item The fine-tuned policies are saved into: \( \Pi_1 \) for \textsc{Vanilla PPO} and \( \Pi_2 \) for \textsc{Variance-Aware PPO}.
\end{enumerate}

\paragraph{Evaluation}  
To assess the effectiveness of the fine-tuned LLMs, we evaluate the policies produced by both \textsc{Vanilla} and \textsc{Variance-Aware} PPO methods, along with the original (reference) LLM:

\begin{itemize}[leftmargin=1.5em, nosep]
    \item For each LLM and reward model setup, we fine-tune 80 policies using \textsc{Vanilla} PPO and another 80 using the proposed \textsc{Variance-Aware} PPO, resulting in a total of 160 fine-tuned policies per model.
    \item These policies are organized into two sets: $\Pi_1$ (Vanilla) and $\Pi_2$ (Variance-Aware).
    \item Additionally, to evaluate the reference model $\mathcal{M}$, we simulate 80 independent runs, generating one response per prompt in each run. These are stored as the reference set $\Pi_0$.
\end{itemize}

The evaluation procedure for each policy $\pi \in \Pi_1 \cup \Pi_2 \cup \Pi_0$ follows Algorithm~\ref{alg:evaluate_finetuned_policies}:
\begin{enumerate}[leftmargin=1.5em, nosep]
    \item For each prompt $x \in \mathcal{D}$:
    \begin{itemize}[nosep]
        \item Generate response $y = \pi(x)$.
        \item Concatenate $s = x \mathbin\Vert y$.
        \item Pass $s$ to the reward model $\mathcal{R}$ via Algorithm~\ref{alg:reward_model}.
        \item Estimate the \textit{true reward} $r^*(s)$ using Algorithm~\ref{alg:true_reward}:
        \begin{itemize}[nosep]
            \item If \textsc{Prompted}, then $r^*(s) = (a + b)/2$, assuming a uniform distribution over $[a, b]$.
            \item If \textsc{Ensemble}, then $r^*(s)$ is computed via an external reference model, \textsc{FsfairX-LLaMA3-RM-v0.1}.
        \end{itemize}
    \end{itemize}
    \item Compute the average of $r^*(s)$ values across all prompts and policies in each set $\Pi_1$, $\Pi_2$, and $\Pi_0$.
\end{enumerate}

The resulting average rewards $\bar{r}_{\Pi_1}$, $\bar{r}_{\Pi_2}$, and $\bar{r}_{\Pi_0}$ capture the alignment effectiveness of each method. These distributions are visualized as histograms to compare alignment quality across models. The base LLM serves as a baseline to assess the benefit of fine-tuning under uncertainty.

%% file: Neurips/PPO_Hardware.tex
\subsection{Hyperparameter and Resource Requirements}
\label{sec:hardware}
\subsubsection{Ensemble Reward Models}

All experiments in this configuration are conducted using \textsc{GPT-2} paired with a custom ensemble reward model, as detailed in Appendix~\ref{sec:reward_model}. This ensemble consists of ten independently initialized scalar reward heads trained on the same frozen foundation model. We train separate policy sets using both \emph{Vanilla PPO} and \emph{Variance-Aware PPO}. The hyperparameter settings for each method are summarized in Tables~\ref{tab:hyperparams_vanilla_ppo} and~\ref{tab:hyperparams_variance_ppo}, respectively. Most hyperparameters follow the design choices in \citet{vonwerra2022trl}. The primary difference between the two approaches is the selection of the $\beta$ parameter (defined as the Initial KL Coeff), which adjusts the KL constraint strength during policy optimization. In our setup, we calibrate $\beta$ such that both variants maintain similar KL divergence to the reference policy, thereby controlling for exploration range.

We also document the hardware and resource usage for training one policy model in Table~\ref{table:hw_requirements_ppo}. Each model was trained on a system equipped with four NVIDIA A40 GPUs (48\,GB each), with total GPU memory consumption averaging 18.4\,GB. The full pipeline—including dataset loading, PPO updates, logging, and checkpointing—requires approximately 6.55\,GB of disk space and completes in roughly 3.86 hours.

\begin{table*}[!htbp]
    \centering
    \begin{minipage}{0.45\textwidth}
        \centering
        \resizebox{0.8\textwidth}{!}{%
        \begin{tabular}{|c|c|}
            \hline
            \textbf{Hyperparameter}    & \textbf{Value} \\ \hline
            Effective Batch Size       & 128             \\ \hline
            Learning Rate              & 1.414e-5        \\ \hline
            Epochs                     & 1               \\ \hline
            Steps                      & 192             \\ \hline
            Initial KL Coeff           & 0.2             \\ \hline
            Adaptive KL Control        & False           \\ \hline
        \end{tabular}
        }
        \caption{\footnotesize Hyperparameters for \textsc{Vanilla PPO} training with ensemble reward models.}
        \label{tab:hyperparams_vanilla_ppo}
    \end{minipage}%
    \hspace{0.04\textwidth}
    \begin{minipage}{0.45\textwidth}
        \centering
        \resizebox{0.75\textwidth}{!}{%
        \begin{tabular}{|c|c|}
            \hline
            \textbf{Hyperparameter}    & \textbf{Value} \\ \hline
            Effective Batch Size       & 128             \\ \hline
            Learning Rate              & 1.5e-5          \\ \hline
            Epochs                     & 1               \\ \hline
            Steps                      & 192             \\ \hline
            Initial KL Coeff           & 0.05            \\ \hline
            Adaptive KL Control        & False           \\ \hline
        \end{tabular}
        }
        \caption{\footnotesize Hyperparameters for \textsc{Variance-Aware PPO} training with ensemble reward models.}
        \label{tab:hyperparams_variance_ppo}
    \end{minipage}
\end{table*}

\begin{table}[!htbp]
\centering
\begin{adjustbox}{max width = 0.5\textwidth}
\begin{tabular}{|c|c|}
\hline
\textbf{Resource}                 & \textbf{Details} \\ \hline
GPU Model                         & NVIDIA A40       \\ \hline
Number of GPUs                    & 4                \\ \hline
Total GPU Memory Used             & 18.4 GB          \\ \hline
Total Disk Space Required         & 6.55 GB          \\ \hline
Total Training Time               & 3.86 hours       \\ \hline
\end{tabular}
\end{adjustbox}
\caption{\footnotesize Hardware requirements for training a single PPO model using \textsc{GPT-2} and a custom ensemble reward model.}
\label{table:hw_requirements_ppo}
\end{table}

\subsubsection{Prompted Reward Models}

We evaluate two LLM configurations using prompted reward models:

\begin{itemize}[leftmargin=1.5em, nosep]
    \item \textsc{Qwen2.5-0.5B} is fine-tuned without any parameter-efficient methods or quantization. The training hyperparameters for both \textsc{Vanilla} PPO and \textsc{Variance-Aware} PPO are listed in Table~\ref{table:quen_hyperparamters}. Hardware configuration and resource requirements for training 80 policies are summarized in Table~\ref{table:hw_qwen_requirements}. Each model requires approximately 77\,GB of GPU memory, 2.9\,GB of disk space, and 30 minutes of training time. All experiments were conducted on a cluster of 4 NVIDIA A100 GPUs.

    \item \textsc{Mistral-7B} is fine-tuned using LoRA-based parameter-efficient techniques and 4-bit quantization via BitsAndBytes. The training hyperparameters for both PPO variants are shown in Table~\ref{table:mistral_hyperparameters}. The LoRA configuration is provided in Table~\ref{table:mistral_peft_config}, and the quantization setup is detailed in Table~\ref{table:quant_mistral_config}. Each model requires approximately 70\,GB of GPU memory, 17\,GB of disk space, and 45 minutes of training time. Training was performed on a cluster of 4 NVIDIA A100 GPUs.
\end{itemize}

\begin{table}[!htbp]
\centering
\begin{minipage}{0.48\textwidth}
\centering
\begin{tabular}{|l|c|}
\hline
\textbf{Hyperparameter} & \textbf{Value} \\
\hline
Batch Size & 16 \\
Mini Batch Size & 4 \\
Learning Rate & 3.14e-9 \\
Gradient Accumulation Steps & 1 \\
Initial KL Coeff & 0.3 \\
Adaptive KL Control & True \\
KL Penalty Type & kl \\
Target KL Divergence & 0.7 \\
Horizon & 19 \\
\hline
\end{tabular}
\caption{\footnotesize Hyperparameters for PPO fine-tuning \textsc{Qwen2.5-0.5B} with prompted reward models.}
\label{table:quen_hyperparamters}
\end{minipage}
\hfill
\begin{minipage}{0.48\textwidth}
\centering
\begin{tabular}{|l|c|}
\hline
\textbf{Hyperparameter} & \textbf{Value} \\
\hline
Effective Batch Size & 64 \\
Mini Batch Size & 16 \\
Learning Rate & 5.7e-5 \\
Gradient Accumulation Steps & 4 \\
Initial KL Coeff & 0.1 \\
Adaptive KL Control & True \\
KL Penalty Type & kl \\
Target KL Divergence & 0.7 \\
Horizon & 19 \\
Clip Range & 0.2 \\
Clip Range (Value Function) & 0.5 \\
Value Function Coeff & 0.3 \\
\hline
\end{tabular}
\caption{\footnotesize Hyperparameters for PPO fine-tuning \textsc{Mistral-7B} with prompted reward models.}
\label{table:mistral_hyperparameters}
\end{minipage}
\end{table}

\begin{table*}[!htbp]
\centering
\begin{minipage}[t]{0.42\linewidth}
\centering
\begin{tabular}{|l|c|}
\hline
\textbf{Resource} & \textbf{Qwen2.5-0.5B} \\
\hline
GPU Model & NVIDIA A100 \\
Number of GPUs & 4 \\
GPU Memory per Device & 77 GB \\
Total Disk Space Required & 2.9 GB \\
Training Time per Model & 30 minutes \\
\hline
\end{tabular}
\caption{\footnotesize Hardware requirements for fine-tuning \textsc{Qwen2.5-0.5B}.}
\label{table:hw_qwen_requirements}
\end{minipage}%
\hfill
\begin{minipage}[t]{0.42\linewidth}
\centering
\begin{tabular}{|l|c|}
\hline
\textbf{Resource} & \textbf{Mistral-7B} \\
\hline
GPU Model & NVIDIA A100 \\
Number of GPUs & 4 \\
GPU Memory per Device & 70 GB \\
Total Disk Space Required & 17 GB \\
Training Time per Model & 45 minutes \\
\hline
\end{tabular}
\caption{\footnotesize Hardware requirements for fine-tuning \textsc{Mistral-7B}.}
\label{table:hw_mistral_requirements}
\end{minipage}
\end{table*}

\begin{table*}[!htbp]
\centering
\begin{minipage}[t]{0.43\linewidth}
\centering
\begin{tabular}{|l|c|}
\hline
\textbf{Configuration Parameter} & \textbf{Value} \\
\hline
LoRA Rank (Low-rank dimension) & 128 \\
LoRA Scaling Factor & 32 \\
LoRA Dropout Rate & 0.05 \\
Bias Adaptation & None \\
\hline
\end{tabular}
\caption{\footnotesize LoRA configuration for fine-tuning \textsc{Mistral-7B} using PPO.}
\label{table:mistral_peft_config}
\end{minipage}%
\hfill
\begin{minipage}[t]{0.43\linewidth}
\centering
\begin{tabular}{|l|c|}
\hline
\textbf{Configuration Parameter} & \textbf{Value} \\
\hline
Enable 4-bit Weight Loading & True \\
4-bit Compute Data Type & bfloat16 \\
Enable Double Quantization & True \\
4-bit Quantization Type & Normal Float 4 (nf4) \\
\hline
\end{tabular}
\caption{\footnotesize BitsAndBytes configuration for loading \textsc{Mistral-7B} with QLoRA.}
\label{table:quant_mistral_config}
\end{minipage}
\end{table*}

The hyperparameters used for response generation during fine-tuning of the LLM are summarized in Table~\ref{table:sampling-hyperparameters}.

\begin{table}[h]
\centering
\begin{tabular}{|l|c|}
\hline
\textbf{Hyperparameter} & \textbf{Value} \\
\hline
Maximum number of generated tokens           & 32   \\
Minimum length of generated sequence         & 8    \\
Top-$k$ sampling cutoff (most likely tokens) & 50   \\
Nucleus sampling threshold ($p$)             & 0.90 \\
Softmax temperature (randomness)             & 1.5  \\
Stochastic decoding enabled                  & Yes  \\
\hline
\end{tabular}
\caption{Sampling hyperparameters used for text generation.}
\label{table:sampling-hyperparameters}
\end{table}

%% file: Neurips/Limitations.tex
\section*{Limitations}

While our variance‐aware RLHF framework addresses an important gap in existing alignment methods, several limitations merit discussion:

\paragraph{Lack of a Universal Definition of `Alignment'}
While this work is notionally situated in the area of methods for alignment of AI models, we must bear in mind that the objective of `aligning' AI models to `human values' lacks a formal and concrete definition. The diversity of human preferences, judgements, biases, intuition and aspirations may very well render the search for a satisfactory definition of alignment impossible. We work only within the narrow technical framework of alignment for AI models that is in force today, based on reducing the measurement of alignment to ideal scalar reward scores, and maximizing these rewards in the hope that they result in well-aligned systems. 

\paragraph{Quality of Variance Estimates}  
Our approach relies on reasonably accurate per–prompt–response variance estimates from the reward model. In the “prompted” setting we assume a uniform distribution over the returned interval, and in the “ensemble” setting we use the empirical sample variance. If these estimates are poorly calibrated (e.g.\ heavy‐tailed errors, systematic bias across prompts), the variance penalty may under‐ or over‐correct, potentially harming performance. Integrating more sophisticated uncertainty quantification (e.g.\ Bayesian ensembling, temperature‐scaled intervals) is left to future work.

\paragraph{Computational Overhead}  
Training with an ensemble of reward heads (or multiple forward passes) incurs nontrivial compute and memory costs relative to single‐model RLHF. Although we mitigate this via shared backbones and parallel heads, scaling to very large reward models (e.g.\ 30B+ parameters) or budget‐constrained settings may be challenging. Exploring low‐cost approximations of the variance penalty (e.g.\ moment‐matching or distilled variance estimators) is an important direction.

\paragraph{Task and Domain Generality}  
Our experiments focus on synthetic interval‐based rewards and open‐source preference datasets (e.g.\ RewardBenchmark’s Chat and Safety tasks). It remains to be seen how variance‐aware fine‐tuning performs on large‐scale, proprietary RLHF setups (e.g.\ summarization, code generation, or multi‐turn dialogue) and in settings where human feedback is collected online.

\paragraph{Risk Metric Scope}  
We define “risk” as the probability of underperforming a reference policy under the true reward model. Other risk measures—such as Conditional Value at Risk (CVaR), worst‐case regret, or percentile‐based bounds—may capture different aspects of policy robustness. Extending our theoretical framework to these alternative metrics could broaden applicability.

\paragraph{Hyperparameter Sensitivity}  
The strength of the variance penalty (\(\beta\) in Equation~\eqref{eq:variance_aware_loss}) must be chosen carefully to balance exploration and conservatism. While we align \(\beta\) with a target KL divergence in our experiments, more automated schemes for selecting this hyperparameter (e.g.\ via validation‐set variance calibration) could improve ease of use.

\medskip

In summary, our work highlights the importance of accounting for reward‐model uncertainty in RLHF, but practical deployment will require addressing variance‐estimation quality, computational cost, and evaluation on a broader range of tasks and risk criteria.  